\documentclass[twoside,11pt]{article}

\usepackage[english]{babel}

\usepackage{amsmath,latexsym,amssymb} 
\usepackage{bbm}

\usepackage{graphicx} 
\usepackage{longtable}
\usepackage{microtype}
\usepackage{siunitx}
\usepackage{hyperref}

\usepackage{subfigure}
\usepackage{booktabs} 

\newcommand{\gKSS}{{{\rm gKSS}}} 

\newcommand{\uk}{\underline{k}} 

\usepackage[utf8]{inputenc}
\usepackage{amsfonts}
\usepackage[english]{babel}
\usepackage{amsthm}
\usepackage{algorithm}
\usepackage{algorithmic}
\usepackage{bbold}
\usepackage{dsfont}

\usepackage{wrapfig}
\usepackage{enumerate}
\usepackage{wrapfig,lipsum,booktabs}

\usepackage[capitalize,noabbrev]{cleveref}

\newtheorem{Definition}{Definition}

\newtheorem{Assumption}{Assumption}

\def\be#1{\begin{equation*}#1\end{equation*}}
\def\ben#1{\begin{equation}#1\end{equation}}

\def\ban#1{\begin{align}#1\end{align}}

\def\given{\typeout{Command 'given' should only be used within bracket command}}
\newcounter{@bracketlevel}
\def\@bracketfactory#1#2#3#4#5#6{
\expandafter\def\csname#1\endcsname##1{%
\addtocounter{@bracketlevel}{1}%
\global\expandafter\let\csname @middummy\alph{@bracketlevel}\endcsname\given%
\global\def\given{\mskip#5\csname#4\endcsname\vert\mskip#6}\csname#4l\endcsname#2##1\csname#4r\endcsname#3%
\global\expandafter\let\expandafter\given\csname @middummy\alph{@bracketlevel}\endcsname
\addtocounter{@bracketlevel}{-1}}%
}
\def\bracketfactory#1#2#3{%
\@bracketfactory{#1}{#2}{#3}{relax}{1mu plus 0.25mu minus 0.25mu}{0.6mu plus 0.15mu minus 0.15mu}
\@bracketfactory{b#1}{#2}{#3}{big}{1mu plus 0.25mu minus 0.25mu}{0.6mu plus 0.15mu minus 0.15mu}
\@bracketfactory{bb#1}{#2}{#3}{Big}{2.4mu plus 0.8mu minus 0.8mu}{1.8mu plus 0.6mu minus 0.6mu}
\@bracketfactory{bbb#1}{#2}{#3}{bigg}{3.2mu plus 1mu minus 1mu}{2.4mu plus 0.75mu minus 0.75mu}
\@bracketfactory{bbbb#1}{#2}{#3}{Bigg}{4mu plus 1mu minus 1mu}{3mu plus 0.75mu minus 0.75mu}
}
\bracketfactory{clc}{\lbrace}{\rbrace}
\bracketfactory{clr}{(}{)}
\bracketfactory{cls}{[}{]}
\bracketfactory{abs}{\lvert}{\rvert}
\bracketfactory{norm}{\Vert}{\Vert}
\bracketfactory{floor}{\lfloor}{\rfloor}
\bracketfactory{ceil}{\lceil}{\rceil}
\bracketfactory{angle}{\langle}{\rangle}

\def\^#1{\ifmmode {\mathaccent"705E #1} \else {\accent94 #1} \fi}
\def\~#1{\ifmmode {\mathaccent"707E #1} \else {\accent"7E #1} \fi}

\edef\-#1{\noexpand\ifmmode {\noexpand\bar{#1}} \noexpand\else \-#1\noexpand\fi}
\def\>#1{\vec{#1}}

\def\atop{\@@atop}

\renewcommand{\leq}{\leqslant}
\renewcommand{\geq}{\geqslant}
\renewcommand{\phi}{\varphi}
\newcommand{\eps}{\varepsilon}

\newcommand{\Var}{\mathop{\mathrm{Var}}\nolimits}

\newcommand{\IR}{\mathbbm{R}}
\newcommand{\IE}{\mathbbm{E}}
\def\tsfrac#1#2{{\textstyle\frac{#1}{#2}}}

\newcommand{\ergm}{{\mathrm{ERGM}}}

\def\s#1{^{(#1)}}

\newcommand{\astar}{a^*}

\newcommand{\R}{\mathbb{R}} 
\renewcommand{\H}{\mathcal{H}} 
\newcommand{\G}{\mathcal{G}} 

\newcommand{\A}{\mathcal{A}} 
\renewcommand{\P}{\mathbb{P}} 
\renewcommand{\Pr}{\P} 
\newcommand{\E}{\mathbb{E}} 
\newcommand{\N}{\mathbb{N}} 
\newcommand{\T}{\mathcal{A}}
\newcommand{\q}{\widehat{q}}

\newcommand{\AgraSSt}{{\rm AgraSSt}}
\newcommand{\gr}[1]{\textcolor{magenta}{#1}}
\newcommand{\wx}[1]{\textcolor{blue}{#1}}

\newcommand{{\steingen}}{{{\rm SteinGen}}}

\usepackage{jmlr2e}

\usepackage{lastpage}
\jmlrheading{xx}{2024}{1-\pageref{LastPage}}{\today; Revised x/xx}{x/xx}{xx-xxxx}{Gesine Reinert and Wenkai Xu}

\ShortHeadings{{\steingen}: 
{Fidelitous and} Diverse Graph Generation}{Reinert and Xu}
\firstpageno{1}

\begin{document}

\title{{{\steingen}}: 
Generating Fidelitous and Diverse Graph Samples}

\author{\name Gesine Reinert \email reinert@stats.ox.ac.uk \\
       \addr Department of Statistics\\
       University of Oxford\\
       Oxford, OX1 3LB, United Kingdom\\
       \AND
       \name Wenkai Xu \email 
       wenkai.xu@uni-tuebingen.de\\
       \addr T\"ubingen AI Center\\
       University of T\"ubingen\\
       72076
       T\"ubingen, Germany 
       }

\editor{}

\maketitle

\begin{abstract}
Generating graphs that preserve characteristic structures while promoting sample diversity can be 
challenging, 
especially when the number of graph observations is small.
{Here}, we tackle the problem of graph generation from {only one observed graph.}

{The classical approach of g}raph generation from parametric {models}  
relies 
on 
{the estimation of parameters,} 
which can be inconsistent or expensive to compute due to 
intractable normalisation constants.
Generative modelling based on 
machine learning techniques to generate high-quality graph samples {avoids} 
parameter estimation but {usually} requires abundant training samples. 
Our proposed generating procedure, 
\emph{{\steingen},} 
{which is phrased in the setting of graphs as realisations of exponential random graph models, combines ideas from Stein's method {and} MCMC by {employing} Markovian dynamics which are based on a Stein operator for the target model. 
{{\steingen}} uses the  Glauber dynamics associated with an estimated Stein operator {to generate a sample}, {and}  re-estimates the Stein operator {from the sample} after every sampling step.
We show that on a class of exponential random graph models this} 
novel ``estimation and re-estimation'' generation strategy {yields high distributional similarity {(high fidelity)} to the original data, combined with high sample diversity.} 
\end{abstract}



\section{Introduction}\label{sec:intro}

Synthetic data generation 
{is} 
a key ingredient for many modern statistics and machine learning tasks
{such as} Monte Carlo tests, enabling privacy-preserving
data analysis, data augmentation, or visualising representative samples. 
{Synthetically generated data can be useful even when in principle the original data set is the only focus of interest, as using the original data for machine learning tasks can be problematic,}
for example when training models on small or imbalanced samples, or even prohibitive  for example due to authority regularisation on privacy-sensitive information; 
see for example \cite{figueira2022survey}.

\clearpage 

{Synthetic data generation learns a procedure to generate samples that capture the main features of an original dataset.} 
%
In particular, 
data in the form of graphs {(or{, used interchangeably,} networks)}
have been 
explored in the machine learning community {to tackle tasks including community detection, prediction and graph representational learning \citep{chami2022machine, abbe2015community, hein2007graph}.}
Viewing 
the observed dataset as a realisation from a
{learnable} probability distribution,  model learning and generating graph samples have been challenging tasks due to the complex 
dependencies within graphs. 
{Statistical}
methods for model fitting and simulation 
from such models 
are available, {{see for example Part III in \cite{newman2018networks} and Chapter 6 in \cite{kolaczyk2009statistical}}.
 {{H}owever, {for complex network models such as exponential random graph models,} intractable normalisation constants can pose} 
a major challenge for parametric modelling, see \cite{handcock2008statnet}.

{Thus} from a computational viewpoint it may be advantageous to assume that edges are {generated} independent{ly}. Edge independent models include the inhomogeneous random graph model by \cite{bollobas2007phase},  and graphon models which originated in \cite{lovasz2006limits}; latent space models introduced in \cite{hoff2002latent}, of which stochastic blockmodels are a special case,  create an embedding in a latent space and then assume that edges occur independently with probabilities described through the latent space, see also  the survey  \cite{sosa2021review}. 
{However,} \cite{chanpuriya2021power} showed that 
{synthetic network} generators {which assume independently generated edges} tend to generate many more triangles and 4-cycles than are present in the data. 

{V}arious deep generative models for graphs have been developed, such as methods using a variational autoencoder (VAE) \citep{simonovsky2018graphvae}; 
using recurrent neural networks ({Graph}RNN) \citep{you2018graphrnn}; based on a generative adversarial network ({Net}GAN) \citep{bojchevski2018netgan} or score-based approaches \citep{niu2020permutation}.
{\cite{goyal2020graphgen} convert networks into  sequences and then use an {Long Short-Term Memory (LSTM) network} to generate samples from these sequences.}  
{DiGress \citep{vignac2023digress} {develops} a diffusion approach with denoising.} A survey on applications of deep generative models for graphs can be found in \citet{guo2022systematic}. 
{While} achieving superior performances in some graph generation tasks and being able to adaptively learn implicit network features,
these deep-learning approaches {typically} rely heavily on a large number of training samples for stochastic optimisation \citep{kingma2014adam}. 
However, {often}
only 
a single graph {is observed}. 
{Only having one observed graph considerably}
limits the advantages and flexibility of many deep generative models on graphs trained via stochastic optimisation. 

Instead, 
these flexible architectures {can be used} 
to sample 
a larger number of subgraphs to create the training set.
For example, \citet{liu2017can} {constructs hierarchical layers of a graph and trains a GAN for each layer.} 
Graph generation based on representation learning and augmentation have also been considered {in} \cite{han2022generative} using re-sampled subgraphs with contrastive learning objectives. 
The CELL method from \cite{pmlr-v119-rendsburg20a} learns 
a probability distribution on networks from a single realisation by using an underlying random walk.  
Like NetGAN, it is however an approach which assumes edge independence. {These ``black-box'' models are hence expected to suffer from the deficiencies pointed out in \cite{chanpuriya2021power}. 
{Although} these methods may reproduce some features of the original data very well, fidelity issues may arise for subgraph counts.}

{While fidelity to the original network is one criterion for synthetic network generators, it is also desirable that the generated synthetic networks show some variability around the original network.}
When there is not much training data available,  
{there is a risk that} graph generation methods may create graphs which are not only similar to each other in their underlying probability distribution, but that are 
actually very similar or even close to identical {to the original graph and between generated samples}, see for example the discussion in \cite{karwa2016dergms}. 
Such samples {may} not reflect the true diversity of the underlying graph distribution and may hence lead to erroneous statistical inference. 
 
{Here,} we focus on the setting that the graph generative model
takes a simple, undirected, unweighted 
{network} 
as input and 
has the task 
to generate synthetic networks 
that 
could be viewed as plausibly coming from the same probability distribution as the one which generated the input network,
while reflecting 
the diversity of 
networks under this probability distribution.}
Motivated by insights from social network analysis \citep{wasserman1994social} we phrase our method in the setting of so-called {\it exponential random graph models  (ERGMs)}. For such models, a  
kernelised goodness of fit test similar to those in \citet{chwialkowski2016kernel, liu2016kernelized} called {\gKSS}~\citep{xu2021stein} 
is available. {We use}
{the proportion of rejected {\gKSS} goodness-of-fit tests 
to assess the quality of 
graph generators.
{For a particular graph sample, the fidelity is assessed via the total variation distance between empirical degree distributions and} 
the diversity 
{is assessed by} 
the pairwise Hamming distance.

\begin{figure}[t]
    \centering    
    \hspace{-0.\linewidth}\includegraphics[width=1.05\linewidth]{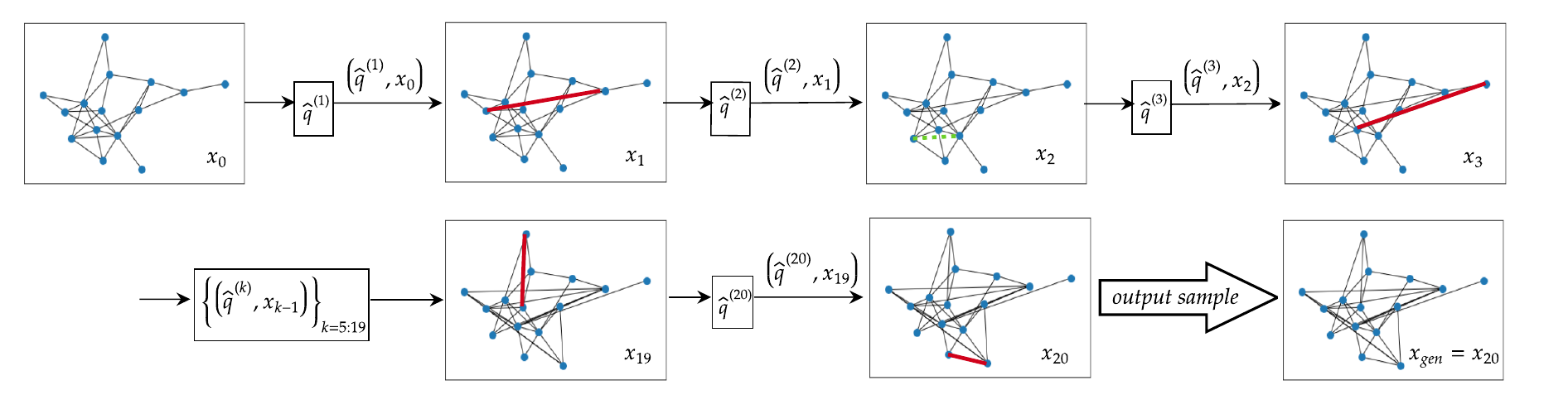}
    \vspace{-.7cm}
    \caption{
    {T}he {\steingen}  
    procedure: 
    $x_0$ is the input network;    in {step $k$} 
 we pick a vertex pair uniformly at random and re-sample its edge indicator from {the} (re-)estimated  conditional probability $\hat{q}^{(k)},$ 
   given the current graph sample $x_{k-1}$ {but excluding the picked vertex pair,}  
   to generate  {the} next graph sample $x_{k}$. 
   Changes in the intermediate steps  are {highlighted}: 
 the thicker {red} solid line {in $x_1, x_3, x_{19}, x_{20}$} denotes that an edge is added, 
 the {green} dashed line {in $x_2$ indicates} 
 {that} an edge is removed.
 {Only samples are shown which differ from the previous sample.}
 %
 The generated sample $x_{20} $ is visually different from the input graph $x_0$.}
   \vspace{-.8cm}
    \label{fig:steingen_illustrate}
\end{figure}

{A key ingredient of {\gKSS} is a Stein operator. This paper proposes to use a Stein operator not only for assessing goodness of fit,
but also for generating graph samples.} 
Inspired by 
Glauber dynamics-based Stein operators for exponential   random graph models \citep{reinert2019approximating},
we propose \textit{{\steingen}}, a novel synthetic sample generating procedure that generates graph samples by running an estimated Glauber dynamics; {in contrast to a classical Markov chain,} the target Glauber dynamics are iteratively re-estimated
{from the current sample.} 
An  
illustration of {{\steingen}}   is shown in \Cref{fig:steingen_illustrate}.
This procedure not only 
{avoids} 
parameter estimation  
{and} dealing with intractable normalising constants but also promotes sample diversity by exploring a rich set of graph configurations.
Moreover, 
{in contrast} to deep models,  
{\steingen} does not require a complicated training phase.
{The procedure is related {to}  MCMC methods, but while MCMC methods would run the same Glauber dynamics for each sample, in  {{\steingen}} the parameters of the Glauber dynamics are updated after every step.} {We also introduce {its classical MCMC version} {\steingen}\textunderscore{nr}, which does not carry out re-estimation after every step {and is hence faster}.} 
{T}he theoretical underpinning for {\steingen}} 
{draws} on approximation results for  
ERGMs 
{from} \cite{reinert2019approximating} {and} {\cite{xu2021stein}.
{In the {spirit} of \cite{xu2022agrasst}, 
{\steingen} could be generalised to input graphs without any underlying {assumptions about a statistical model which generated the graph;}  {the estimation procedure is then related to the one in \cite{bresler2017learning} where Glauber dynamics is used to estimate an undirected graphical model from data.} 
{However,} to illustrate and assess the performance of 
{\steingen} 
we here we concentrate on the ERGM setting.}





{This paper is organised as follows. 
\Cref{sec:background} 
{gives} notation and 
background {on generation of ERGMs, an ERGM Stein operator, the graph kernel Stein statistic {gKSS}, generalisations to other random graphs and the approximate graph Stein statistic {\AgraSSt} as a non-model based goodness-of-fit statistic}. Our 
{\steingen} procedure for graph generation  {is presented} in \Cref{sec:method}, {with} 
theoretical {guarantees}
for {\steingen} regarding consistency, diversity and mixing time given in \Cref{sec:analysis}. {\Cref{sec:analysis}
also introduces the total variation distance between empirical degree distributions as a measure of sample fidelity. 
Numerical results on simulation studies and {a} real network data {case study on a teenager friendship network described in \citet{steglich2006applying}} are provided in  \Cref{sec:exp}. {\Cref{sec:exp} also contains figures of the Hamming distance against (1 minus the total variation distance between the empirical degree distributions) to illustrate the performance of {\steingen} as well as CELL and NetGAN regarding fidelity and diversity.}
{A concluding discussion is found}
in \Cref{sec:conclusion}. 
The Appendix contains 
{more details on parameter estimation as well as additional experiments} on synthetic and real data sets{, including 
Padgett's Florentine marriage network \citep{padgett1993robust}, and protein-protein interaction networks for the Epstein-Barr virus \citep{hara2022comprehensive} and for yeast \citep{von2002comparative}}.
The code for the experiments is 
available at \url{https://github.com/wenkaixl/SteinGen_code.git}.

\section{Assessing the quality and diversity of graph samples}\label{sec:background}


\paragraph{Notation.}
{First we introduce some {notation}. We denote by  $\G^{lab}_n$ the set of vertex-labeled graphs on $n$ vertices, with $N =n(n-1)/2 $ possible undirected edges, and we encode $x \in \G^{lab}_n$ by an ordered collection of $\{0,1\}$-valued variables 
$x = (x^{(ij)})_{1 \le  i < j \le n} \in \{0,1\}^N$,  {where} $x^{(ij)}=1$ {if and only if}
there is an edge between $i$ and $j$. 

We denote an (ordered) vertex-pair 
$s=(i,j)$  by $s\in [N]:=\{1, \ldots, N\}$. 
Let $e_s \in  \{0,1\}^N$ be a vector with $1$ in 
 coordinate {$s$} and 0 in all others; 
$x^{(s,1)} = x + (1-x{^{(s)}}) e_s$ {has the $s$-entry  replaced of $x$ by the value 1, and} $x^{(s,0)} = x - x{^{(s)}} e_s$ {has  the $s$-entry  of $x$ replaced by the value 0; moreover,}
${x}_{- s}$ is the set of edge indicators with  entry $s$ removed. 
More generally, for a graph $H$, {its vertex set is denoted by $V(H)$ and its edge set is denoted by $E(H)$.}  

For $V(H) \le n$
and for $x\in\{0,1\}^N$, denote by $t(H,x)$  the number of
{\it edge-preserving} injections from $V(H)$ to $V(x)$; an injection $\sigma$ preserves edges if for all edges $vw$ of $H$ {with  $\sigma(v)<\sigma(w)$}, $x_{\sigma(v)\sigma(w)}=1$
(here assuming $\sigma(v)<\sigma(w)$). For  $v_H =| V(H)| \ge 3$  set
\begin{equation} \label{eq:th}
t_H(x)=\frac{t(H,x)}{n(n-1)\cdots(n-v_H+3)}.
\end{equation} For $e\in E(H)$, define the graph $H_{- e}$ 
to be $H$ with the edge $e$ removed, but retaining all vertices.
For $h: \{0,1\}^N \rightarrow \R$ we  let
$\Delta_s h(x) = h( x^{(s,1)}) - h(x^{(s,0)})$ 
and  $|| \Delta h|| = \max_{s \in [N]} \max_{x \in \mathcal{G}_n^{\rm{lab}}} | \Delta_s h (x)|.$  The Hamming distance between two graphs $x$ and $y$ is 
\begin{equation} \label{eq:hamm} 
d_H(x,y) = \sum_{s =1}^{N} | x^{(s)} - y{^{(s)}}|.\end{equation} 
The total variation distance $d_{TV}$ between two distributions $P$ and $Q$ on $\{0, 1, 2, \ldots\} $ is
\begin{equation}\label{eq:dtv}
d_{TV}(P,Q) = \frac12 \sum_{k=0}^\infty| P(\{k\}) -     Q(\{k\})|. 
\end{equation}
For a distribution $q$ and a function $f$  the expectation is $\E_{{q}} f = \E f(X) $ 
where $X $ has distribution ${q}$.} 
Vectors in $\R^L$ are  column vectors; the superscript $\top$ denotes the transpose. {The function ${\mathbb 1} (A)$ is the indicator function which equals 1 if $A$ holds, and 0 otherwise.} {The norm $\parallel \cdot \parallel_p$ denotes the $L_p$-norm.}

\subsection{Exponential random graphs}


{Exponential random graph models} (ERGMs) 
have been extensively studied in social network analysis \citep{wasserman1994social,holland1981exponential}; a special case are Bernoulli random graphs.
{F}ix $n\in\mathbbm{N}$ and $k$ connected  graphs $H_1,\ldots,H_L$ with $H_1$ a single edge, and for $\ell=1,\ldots,L$ abbreviate $v_\ell:=\abs{V(H_\ell)}$ (so $v_1=2$).
{Recalling \eqref{eq:th} let}
\begin{equation}\label{eq:thl}
t_\ell(x)=\frac{t(H_\ell,x)}{n(n-1)\cdots(n-v_\ell+3)}.
\end{equation}
If $H{=H_1}$ is a single edge, then $t_H(x)$ is twice the number of edges of $x$. In the exponent this scaling of counts matches Definition~1 in \citet{bhamidi2011mixing} and Sections~3 and~4 of \citet{chatterjee2013estimating}.
An ERGM {
as a random graph model 
for the collection  $x\in \{0,1\}^{N}$
can be defined
as follows.
see \cite{reinert2019approximating}.
\begin{definition}[Definition 1.5 in
\citet{reinert2019approximating}]
\label{def:ergm} 
Fix $n\in \N$ and $L \in \N$. Let $H_1$ be a single edge and for $l=2, \ldots, L$ let  $H_l$ be a connected graph on at most $n$ vertices. With the notation \eqref{eq:thl}, 
for  $\beta = (\beta_1, \dots, \beta_L)^{\top} {\in \R^L}$ 
we say that 
$X\in \G^{lab}_n$ follows  the exponential random graph model  $X\sim \operatorname{ERGM}(\beta, t)$ if for  {for all} $x\in \G^{lab}_n$,
\begin{equation}\label{eq:ergm}
    {\Pr} (X = x) = \frac{1}{\kappa_n(\beta)}\exp{\left(\sum_{l=1}^L \beta_l t_l(x) \right)}.
\end{equation}
\end{definition}
In \eqref{eq:ergm}, $\kappa_n(\beta)$ is a  normalisation constant, {which even for moderately sized graphs is usually intractable}.
When $L=1$ then  $\ergm (\beta)$ has the same distribution as an Erd\"os-R\'enyi (ER) graph 
with parameter $p$, in which edges appear independently with probability $p = e^{\beta} / ( 1 + e^(\beta).$

Parameter estimation $\hat \beta$ for $\beta$ is 
{only} possible when the 
statistic  
${t(x)=(}t_1 (x), \ldots, t_L(x)  {)}${, which is a}  sufficient statistic {
for the parameter $\beta=(\beta_1, \ldots, \beta_L)^\top$,} is 
specified a priori.
%
As the normalising constant $\kappa_n(\beta)$ is usually intractable, often 
Markov Chain Monte Carlo (MCMC) procedures are used for 
MLE-type parameter estimation \citep{snijders2002markov}; {these are abbreviated} 
 MCMCMLE.
{A computationally efficient but often less accurate  to MCMCMLE is provided by the} Maximum Pseudo-Likelihood Estimator (MPLE) \citep{besag1975statistical},  
\citep{strauss1990pseudolikelihood,schmid2017exponential}.  
Contrastive divergence \citep{hinton2002training} has also been used for parameter estimation in 
ERGMs  \citep{hunter2006inference}.
Additional material on parameter estimation methods can be found in \Cref{app:parameter}.

Apart from specific models such as ER graphs or an Edge-2Star (E2S) model \citep{mukherjee2013statistics}, parameter estimation for $\beta$ in \eqref{eq:ergm} may not 
be consistent 
\citep{shalizi2013consistency}. Convergence results for restricted exponential family models are discussed {in} \cite{jiang2018convergence}.
For an ERGM  with  specified sufficient statistic 
$t(x)$,
parameter estimation and 
graph generation 
is 
implemented in the R package 
\texttt{ergm} \citep{morris2008ergm}, see also  \cite{handcock2008statnet}; {this implementation is used as a baseline} 
for our investigation.


\subsection{Glauber dynamics and
Stein operators {for ERGMs}}\label{sec:glauber}

{\steingen} is based on Glauber dynamics and a Stein operator for ERGMs. {To explain these notions, we start with a Stein operator. For a probability distribution $q$ on a measureable space ${\mathcal{X}}$,}
an operator $\T_q$ acting on functions  $f: {\mathcal{X}}\to \mathbb{R}^d$ {for some $d$} is called a {\it Stein operator}
 {with {\it Stein class}  ${\mathcal F}$ of functions $f: {\mathcal{X}}\to \mathbb{R}^d$}
 if for all $f \in {\mathcal F} $ the so-called {\it Stein identity} holds: $\E_q[{\T}_q f] = 0$. 
 {A particular instance of such a Stein operator is an infinitesimal operator of a Markov process which has the target distribution $q$ as unique stationary distribution, see for example \cite{barbour1990stein}. 
{In \cite{reinert2019approximating}, it was shown that a suitable Markov process for $q$ the distribution of an $\ergm(\beta, t)$ given in \eqref{eq:ergm} is  provided by} Glauber dynamics  on $ \{0,1\}^N$,  
with transition probabilities 
\begin{equation}\label{eq:markov_transition}
    \P (x \rightarrow x^{(s,1)} ) = \frac1N -
 \P(x \rightarrow x^{(s,0)}) = \frac1N q_X({s,1} | x{_{-s}}),
\end{equation}
where
$q_X({s,1} | {x_{-s}}) := \P(X{^{(s)}}=1|{X_{-s}=x_{-s}}).$
{From \eqref{eq:ergm},}
$$
q_X({s,1}|{x_{-s}}) = 
\frac{\exp\left\{\sum_{\ell=1}^L \beta_\ell t_\ell(x^{(s,1)})\right\}} 
{\exp\left\{ \sum_{\ell=1}^L \beta_\ell t_\ell(x^{(s,1)})\right\} +\exp\left\{\sum_{\ell=1}^L \beta_\ell t_\ell(x^{(s,0)})\right\} }.
$$
As   
$\Delta_s t_\ell (x)
= t_\ell(x^{(s,1)}) - t_\ell(x^{(s,0)})$ depends only on $x_{-s}$,
cancelling out common factors, \begin{align}
q({s,1}| {x_{-s}} ) =  \exp\left\{\sum_{\ell=1}^L \beta_\ell {\Delta_s t(x)}\right\} 
\left( \exp\left\{ \sum_{\ell=1}^L \beta_\ell {\Delta_s t(x)}\right\} +1 \right)^{-1}  
=:
q({s,1} | {\Delta_s t(x)} ).
\label{eq:condexp}
\end{align}
Thus, the transition probability in \eqref{eq:markov_transition}} depends only 
on $\Delta_s t(x)$. 
Similarly, 
exchanging $1$ and $0$ in this formula gives $q({s,0}|{x_{-s}} )$. 
{The Stein operator from \cite{reinert2019approximating} is the} generator $\T_{\beta, {t}} $ of this Markov process; 
\begin{eqnarray}\label{eq:stein_operator}
    \T_{\beta, {t}} 
    = \T_{q}f(x)
=  \frac{1}{N} \sum_{s\in[N]} \A^{(s)}_q f(x)
\end{eqnarray}
with summands
\begin{eqnarray}\label{eq:stein_component}
    \A^{(s)}_q f(x) &=& {q({s,1} | {\Delta_s t(x)} )} \, \Delta_s f(x) + \left( f(x^{(s,0)}) - f(x)\right)
  .
\end{eqnarray}
{\begin{lemma} Each operator given in \eqref{eq:stein_component}satisfies the {\it Stein identity;} for each $s \in [N],$
\begin{eqnarray}\label{meanzeroproperty} 
\mathbb{E}_{q} \T_{q}^{(s)}f = 0.
\end{eqnarray}
\end{lemma}
}
\begin{proof}
By conditioning {and using that $q({s,1}|x_{-s})= q({s,1}|\Delta_s t (x) )$},  
\begin{eqnarray*}
\mathbb{E}_{q} \T_{q}^{(s)}f 
&= &
\sum_x q(x_{-s}; x^{(s,1)}) \T_{q}^{(s)}f(x^{(s,1)})  + q(x_{-s}; x^{(s,0)}) \T_{q}^{(s)}f(x^{(s,0)})  \\
&=& 
{\sum_x q( x_{-s})\left(q({s,1} | \Delta_s t(x) )\T_{q}^{(s)}f(x^{(s,1)}) +q({s,0} | \Delta_s t(x) )\T_{q}^{(s)}f(x^{(s,0)}) \right).}
\end{eqnarray*}
{Here we used that $\Delta_s f(x^{(s,1)}) =  \Delta_s f(x^{(s,0)}) = \Delta_s f(x)$ does not depend on $x^{(s)}$.}
Substituting $x^{(s,1)}$ and $x^{(s,0)}$ in  \eqref{eq:stein_component} 
gives \begin{eqnarray*}
\T_{q}^{(s)}f(x^{(s,1)}) &= &q({s,1} | {\Delta_s t(x)} )
\Delta_s f(x) + (f(x^{(s,0)}) - f(x^{(s,1)}))\\ 
&= &(1 - q({s,0} | {\Delta_s t(x)} )) 
\Delta_s f(x) + (f(x^{(s,0)}) - f(x^{(s,1)}))\\
&=& - q({s,0} | {\Delta_s t(x)} ) 
\Delta_s f(x)
\end{eqnarray*} and $
\T_{q}^{(s)}f(x^{(s,0)}) = q({s,1} | {\Delta_s t(x)} )
\Delta_s f(x) . $
Thus, 
\begin{eqnarray*}
\lefteqn{q({s,1} | {\Delta_s t(x)} )\T_{q}^{(s)}f(x^{(s,1)}) +q ({s,0} | {\Delta_s t(x)} )\T_{q}^{(s)}f(x^{(s,0)})}\\
&=-q({s,1} | {\Delta_s t(x)})  q({s,}0| {\Delta_s t(x)} )\Delta_s f(x) +q({s,0}| {\Delta_s t(x)}  ) q({s,1}| {\Delta_s t(x)} )\Delta_s f(x) = 0.
\end{eqnarray*}
\end{proof}

Under suitable conditions,  the  ERGM 
Stein operator in \eqref{eq:stein_operator} is close to the  $G(n,p)$ Stein operator, see 
\citet{reinert2019approximating}, Theorem 1.7, with details provided in the proof of Theorem 1 in  \citet{xu2021stein}. To state the result, a technical assumption is required, which originates in 
 \citet{chatterjee2013estimating}.
{With the notation in Definition~\ref{def:ergm}  for ERGM$(\beta, t)$}, {for $a\in[0,1]$ we set}
$\Phi(a) := \sum_{\ell=1}^{L} \beta_\ell e_\ell a^{e_\ell -1},$ and $ 
\varphi(a) 
:= ({1+ \tanh(\Phi(a))})/{2}, 
$
where $e_\ell$ is the number of edges in $H_\ell$. For a polynomial $f(x) = \sum_{i=1}^k a_k x^k$ we use the notation
$|f|(x) = \sum_{i=1}^k |a_k | x^k$. 

\begin{Assumption}\label{assum:er_approx}
There is a unique $a^{*} \in [0,1]$ that solves 
$\varphi(a^{*}) = a^{*}$; moreover $\frac{1}{2}|\Phi|'(1) < 1$.
\end{Assumption}
{Such a value $a^*$ will be used} as  edge probability in an approximating Bernoulli random graph, $\operatorname{ER}(a^*)$. The following result holds. 

\begin{proposition}[\citet{xu2022agrasstarxiv}Proposition A.4] \label{ergmconvergence} 
Let $q(x)=\operatorname{ERGM}(\beta, t)$ {satisfy} Assumption \ref{assum:er_approx} and let ${\tilde q}$ denote the distribution of {ER$(a^*)$.}
Then there is an explicit constant $C=C(\beta, t, {K})$  such that for all $\epsilon > 0,$ 
$  \frac{1}{N} \sum_{s \in {N}} \E | (\A_q^{(s)}  f(Y) - \A_{\tilde{q} } ^{(s)} f(Y) ) | \le || \Delta f|| {n \choose 2} \frac{C(\beta, t)}{\sqrt{n}}.$
\end{proposition}

{The behaviour of Bernoulli random graphs is relatively well understood due to the independence of the edge indicators in this model. Many of the theoretical guarantees in this paper are based on first showing that the ERGM in question is close to a suitable Bernoulli random graph, and then deriving the guarantee in question for the Bernoulli random graph.}

\subsection{The graph kernel Stein statistic \gKSS}\label{sec:gkss} 

{
{B}ased on the heuristic that if a distribution $p$ is close to $q$ then $\E_p[\T_q f(x)]\approx 0$, {the quantity} 
$\sup_{f \in {\mathcal F}}  | \E_p[\T_q f(x)]| $ can be used to assess 
{a} distributional distance between $q$ and $p$. The choice of ${\mathcal F}$ is crucial for {making this quantity computable;}
 see \cite{gorham2015measuring}. In \cite{chwialkowski2016kernel} and \cite{liu2016kernelized} it was suggested to use as ${\mathcal F}$ the unit ball of a reproducing kernel Hilbert space (RKHS). 
 A corresponding distributional difference measure, {the  {\it graph kernel Stein statistic (gKSS)}} based on the ERGM Stein operator \eqref{eq:stein_operator}, is introduced in \citet{xu2021stein} to perform a goodness-of-fit testing procedure for \emph{explicit} exponential random graph models 
 {even when only a single network is observed}. 
{F}or a fixed graph  $x$, {and an RKHS $\mathcal H$},
{to test goodness-of-fit to a $q=\ergm(\beta, t)$ distribution,} gKSS is defined as
\begin{align}\label{eq:gkss}
    \operatorname{gKSS}(q;x) 
    & = \sup_{\|f\|_{\H}\leq 1} \Bigg|
    {\frac{1}{N} \sum_{s\in[N]} \A^{(s)}_{{q,t}} f(x,)} 
     \Bigg|,
\end{align}
where the function $f$ is chosen to best distinguish $q$ from $x$.
For {an} RKHS $\H$ associated with kernel $K$, 
{by the reproducing property of $\H$,} 
the squared version of gKSS admits 
{an explicit} quadratic form representation {which can be readily computed,}
\begin{equation}
\label{eq:gkss_quadratic}    
\operatorname{gKSS}^2(q;x) = \Bigg\langle {\frac{1}{N} \sum_{s\in[N]} \A^{(s)}_{{q,t}} K(x,\cdot), \frac{1}{N} \sum_{u\in[N]} \A^{(u)}_{{q,t}} K(x,\cdot)
}
\Bigg\rangle.
\end{equation}

\subsection{Beyond ERGMs}
\label{sec:agrasst}

{While gKSS is only available for $\ergm$s,}
{in practice, {instead of assuming an ERGM,}
as in \cite{xu2022agrasst}, in \eqref{eq:markov_transition} we {could} use more general conditional probabilities,} 
based on network statistics. 
Let $t(x)$ be {a (possibly vector-valued)} network statistic which takes on finitely many values $\uk$,  and let {$ q_{\uk}( {s,1}|\Delta_s t (x) =\uk) = 
\mathbb{P}(X^{{(s)}} =1 | {\Delta_s} t(x) = \uk)$}; we assume that $q_{\uk} ( x) > 0$ for all $\uk$ under consideration. 
In analogy with 
\eqref{eq:condexp}, we introduce a Markov chain on $\G^{lab}_n$
which  transitions from $x$ to $x^{(s,1)}$ with probability 
\begin{equation}\label{eq:cond_prob}
    q_{\uk}({s,1}|\Delta_s t (x))  = \mathbb{P} (X^s = 1 | {\Delta_s t(x)} )
    ,
\end{equation} 
and 
from $x$ to 
$x^{(s,0)}$ with probability $q({s,0} | {\Delta_s t(x)}  )
= 1- q_t ({s,1} {| {\Delta_s t(x)}  }) ;$ no other transitions occur.
The corresponding Stein operator {is}
$
    \T_{{q}, {t}} f(x)
=  \frac{1}{N} \sum_{s\in[N]} \A^{(s)}_{{q,t}} f(x) 
$
with
\begin{align}
\mathcal{A}_{q,t}^{(s)} f(x) 
 =  
 q ({s,1}{| {\Delta_s t(x)}  }) f( x^{(s,1)})  + q({s,0}{| {\Delta_s t(x)}  })f(x^{(s,0)})  - f(x).
\label{eq:cond_stein}
\end{align}
For an ERGM,  $t(x)$ could be taken as a  vector of the sufficient statistics 
but here we do not even assume a parametric network model $q(x)$, and
$t(x)$ is specified by the user. 

If there is no  
closed-form conditional probability ${q({s,1}| {\Delta_s t(x)}  )}$ 
in \eqref{eq:cond_prob} {available}, 
the Glauber dynamics in \eqref{eq:markov_transition} 
can be 
carried out for an
estimated conditional distribution $ {\widehat q({s,1} | {\Delta_s t(x)}  )}$.
To compare the 
estimated model $ {\widehat q(x^{{(}s,1{)}} | {\Delta_s t(x)}  )}$ 
and the sample $x$, the
Approximate graph Stein statistic (AgraSSt) 
{from}
\citet{xu2022agrasst}  
takes functions in {an} appropriate RKHS
to 
distinguish the model from the data, and is defined as 
\begin{equation} \label{eq:agrassteq} \operatorname{AgraSSt}(\widehat q, t;x) 
 = \sup_{\|f\|_{\H}\leq 1} \Big|N^{-1}\sum_s \A^{(s)}_{\widehat q,t} f(x)\Big|.\end{equation}
Due to the reproducing property of {the} RKHS, 
AgraSSt admits a quadratic form,
\begin{equation}
{\rm{AgraSSt}}^2 (q; x) = N^{-2}  \sum_{s \in [N]} \sum_{s'\in [N]} 
\left\langle \A^{(s)}_{\widehat q,t}  K(x,\cdot), \A^{(s')}_{ \widehat q,t}  K(\cdot,x)\right\rangle_{\H}.
    \label{eq:agrasst_quadratic}
\end{equation}
In practice, $N$ can be large and AgraSSt takes $N^2$ {steps} to compute the double sum, which can be computationally inefficient. 
\citet{xu2022agrasst} considers an edge re-sampled form that improves the
computational efficiency;
it is given by 
\begin{equation}\label{eq:agrasst}
\widehat{\operatorname{AgraSSt}}(\widehat q, t;x) 
= B^{-2}\sum_{b,b'\in [B]} 
\left\langle \A^{(s_b)}_{\widehat q,t}  K(x,\cdot), \A^{(s_{b'})}_{ \widehat q,t}  K(\cdot,x)\right\rangle_{\H}.
\end{equation}
{{While} in principle, any multivariate statistic $t(x)$ can be used in this formalism, 
estimating the conditional probabilities using relative frequencies can 
be computationally prohibitive. 
Instead, here we consider simple summary statistics, such as} edge density, degree statistics or {the} number of neighbours connected to both vertices of $s$.  
The estimation procedure {for the transition probabilities} is presented in Algorithm \ref{alg:est_conditional}  {which is adapted from \cite{xu2022agrasst} {by estimating} 
the conditional probability 
using only one network.} 

\begin{algorithm}[t]
   \caption{Estimating {the} conditional probability {${\widehat q(x^{{(}s,1{)}} | {\Delta_s t(x)}  )}$}}
   \label{alg:est_conditional}
\begin{algorithmic}[0]
\renewcommand{\algorithmicrequire}{\textbf{Input: }}
\renewcommand{\algorithmicensure}{\textbf{Procedure:}}
\REQUIRE~~\\
network $x$; 
network statistics $t(\cdot)$;
\ENSURE~~\\
\STATE Estimate the conditional probability ${{q}}({s,1}|{\Delta_s t(x)} =\uk)$ of the edge $s$ being present conditional on ${\Delta_s t(x)} =\uk$ by 
the relative frequency ${\widehat{q}}({s,1}|{\Delta_s t(x)} =\uk)$ of an edge at $s$ when ${\Delta_s t(x)} =\uk$.
\renewcommand{\algorithmicensure}{\textbf{Output:}}
\ENSURE~~\\
${\widehat{q}}({s,1}|{\Delta_s t(x)} =\uk)
$
that estimates 
$q({s,1}|{\Delta_s t(x)} =\uk)$ in \eqref{eq:cond_prob}.
\end{algorithmic}
\end{algorithm}

\section{{\steingen}: generating {fidelitous} 
graph 
 samples
with diversity}\label{sec:method}

The idea {behind} {\steingen} is {as follows.}
If $\mathcal{A}$ is the generator of a Markov process $(X_t, {t \ge 0})$ with unique stationary distribution $\mu$  
then, {under regularity conditions}, 
running the Markov process from an initial {distribution},
$X_t$ converges to the stationary distribution {in probability as $t \rightarrow \infty$}.  In particular, Glauber dynamics as in \eqref{eq:markov_transition} preserves the stationary distribution. Thus, the original sample together with the sample after one step of the Glauber dynamics can be used to re-estimate the transition probabilities given by \eqref{eq:cond_prob}. 
This idea is translated into the {\steingen} procedure as follows.} 
\begin{enumerate}
    \item 

We estimate the conditional probability $ {q({s,1} | {\Delta_s t(x)})}$ {from the observed graph $x$} {using \Cref{alg:est_conditional}}; 
denote the estimator as ${\widehat q({s,1} | {\Delta_s t(x)})}$.
\item Given the current graph $x$ we pick a vertex pair $s \in [N]$ uniformly at random and replace $x^{(s)}$ by $(x^{(s)})'$ drawn 
{to equal 1 with probability ${\widehat q({s,1} | {\Delta_s t(x)})}$, and 0 otherwise.} 
Keeping all other edge indicators as in $x$ 
{results in} a new graph $x'$ which differs from $x$ by at most one edge indicator. 
\item Starting with this new graph $x'$, we estimate ${q({s,1} | {\Delta_s t(x')}  )}$, draw a vertex pair, and replace it by {an edge indicator} drawn from the re-estimated conditional distribution, {again estimated using \Cref{alg:est_conditional}}.
\item This procedure is iterated $r$ times, which $r$ chosen by the user.
\end{enumerate}
The {\steingen} procedure is illustrated in \Cref{fig:steingen_illustrate} {and the algorithm is given in \Cref{alg:steingen}}.
{The fact that} $\mathcal{A}$ is a Stein operator for {the distribution $q$ of an ERGM will be used 
to obtain theoretical guarantees. 
{We end this section with some remarks on the {\steingen} procedure.}


\begin{algorithm}[t]
   \caption{{The} {\steingen} procedure {for generating one network sample}}
   \label{alg:steingen}
\begin{algorithmic}[1]
\renewcommand{\algorithmicrequire}{\textbf{Input:}}
\renewcommand{\algorithmicensure}{\textbf{Objective:}}
\REQUIRE~~\\
    The observed network 
    $x$; network statistics $t(\cdot)$;  
    {number of {steps} $r$ to be executed}
\ENSURE~~\\
Generate 
{one network} sample 
\renewcommand{\algorithmicensure}{\textbf{Procedure:}}
\ENSURE~~\\
\STATE Set $x(0) = x$.
\FOR{$i=1:{{r}} $}
    \STATE  Uniformly sample a vertex pair $s\in[N]$
    \STATE Estimate the conditional distribution {$\widehat {q}({s,1} | {\Delta_s t(x(i-1))}  )$}
{using} \Cref{alg:est_conditional}.
    \STATE With probability  {$\widehat {q}({s,1} | {\Delta_s t(x(i-1))}  )$} set 
    $x(i)^{(s)}=1$; 
    otherwise, set $x(i)^{(s)}=0$.
    \STATE Set $x(i)^{(s')}  = x({i-1})^{(s')},$  for all $ s' \in[N], s' \ne s$. 
    \STATE If $x{(i)} \ne x{(i-1)}$, re-estimate 
    {$\widehat{q}({s,1}| {\Delta_s t(x(i))}  )$} using \Cref{alg:est_conditional};
\ENDFOR
\STATE Record {{$x({r})$}} 
as the 
generated sample
\renewcommand{\algorithmicrequire}{\textbf{Output:}}
\REQUIRE~~\\
The 
generated 
{network} sample {$x(r)$}
\end{algorithmic}
\end{algorithm}

\begin{remark}
\begin{enumerate}
\item 
Direct estimation {of the} conditional probability {using \Cref{alg:est_conditional}} avoids the {often} intractable normalising constant involved in parameter estimation.
%


\item 
A standard MCMC method estimates the Glauber dynamics transition probabilities only once. 
As $q({s,1}| {\Delta_s t(x)}$ is estimated from 
only one graph, {the standard MCMC sampler may not explore the sample space very well.}
The re-estimation steps in {\steingen} increase the variability. 
\item {We} 
also propose a variant, {{\steingen} with \underline{n}o \underline{r}e-estimate}  ({\steingen}\textunderscore{nr}), which estimates  the target ${{q({s,1} | {\Delta_s t(x)})}}$ only once, from the input graph $x$, and then proceeds via Gibbs sampling starting from {$x$}.
{This {variant} 
differs from the MCMC procedure 
in the {\tt{R}} packages \texttt{sna} and \texttt{ergm}, which  uses {$x$} 
only for parameter estimation and then generates samples using the Markov chain with the estimated parameters}.
\item {A guideline for choosing the number $r$ of {steps} is 
$r= N \log N + \gamma N + \frac12$, where $\gamma$ is the Euler-Mascheroni constant, } 
as will be derived in  Subsection \ref{subsec:mixing}.
{Similarly to MCMC procedures, one could alternatively add a stopping rule which depends on the observed difference  between sample summaries.}
\end{enumerate}
\end{remark}

\section{Theoretical analysis}\label{sec:analysis}



In this section we give theoretical guarantees under which, first, {\steingen} 
{is fidelitous in the sense  that it} generates networks from approximately the {correct} 
distribution (\Cref{subsec:consistency}),  and second, {we give guarantees on the diversity of} the resulting networks 
{(\Cref{subsec:diversity})}. 
{\Cref{subsec:mixing} discusses the mixing time of {\steingen}, whereas \Cref{subsec:stability} addresses the stability of the network generation.} {We start  with a result that underpins}  the {\steingen} procedure, showing that the Glauber dynamics preserves its underlying  $\ergm(\beta, t)$ distribution.
\begin{proposition}\label{prop:stat} 
If $X$ follows the $\ergm(\beta, t)$ distribution {and if the corresponding Glauber Markov process is irreducible,} then any sample from its Glauber dynamics \eqref{eq:markov_transition} also follows the $\ergm(\beta, t)$ distribution. 
\end{proposition}

\begin{proof} It is shown in  {Lemma 2.3 of \cite{reinert2019approximating} that {under the assumptions of \Cref{prop:stat},} 
{$\ergm(\beta, t)$} is 
the stationary distribution of {its Glauber Markov process}.
Thus, when started from the stationary distribution, $X \sim$ $\ergm(\beta, t)$, then {at every time $s > 0$ the state $X(s)$ of the} Glauber Markov process  has distribution $\ergm(\beta, t)$.}
\end{proof}
{From here onwards we make the standing assumption that the $\ergm(\beta, t)$ distribution is such that the corresponding Glauber Markov process is irreducible.} 

\subsection{Consistency of the estimation} \label{subsec:consistency}

In 
{{\steingen}}  we estimate the transition probabilities from the sampled network by counting. Our theoretical justification of this procedure 
holds in the so-called {\it high {temperature} regime}, as follows.
{We recall} the definition for $\ergm(\beta, t)$ in \eqref{eq:ergm} {and Assumption \ref{assum:er_approx}.} 
\begin{proposition}\label{prop:consistent} Let $q(x)=\operatorname{ERGM}(\beta, t)$ {satisfy} Assumption \ref{assum:er_approx}. 
For $x$ a realisation of  $\ergm(\beta, t)$, let $N_{\uk}(x) {= \sum_{s \in [N]} 
{\mathbbm 1}(\Delta_s t(x) = \uk)}$ be the number of vertex pairs $s\in [N] $ such that $\Delta_s t(x) = \uk$, and let $n_{\uk}(x){= \sum_{s \in [N]} x^{(s)} {\mathbbm 1}(\Delta_s t(x) = \uk)}$ be the number of vertex pairs $s\in [N] $ such that $\Delta_s t(x)= \uk$ and $s$ is present in $x$.  Then 
$ \widehat{q}({s,1}| \Delta_s t(x) = \uk)
= \frac{n_{\uk}(x) }{N_{\uk}(x)} {\mathbbm{1}(N_{\uk} \ge 1)} $ 
is a consistent estimator of $ {q}({s,1} | \Delta_s t(x) = \uk)$ as $n \rightarrow \infty.$
\end{proposition}

\begin{proof}  Let $\astar$ be as in Assumption \ref{assum:er_approx}; {let $X \sim \ergm(\beta, t)$ and  $Z \sim$ ER$(a^*)$}.  Theorem 1.7 from \cite{reinert2019approximating} gives that, 
for any 
 $h:\{0,1\}^{\binom{n}{2}}\to \IR$, we have
\begin{eqnarray} \label{eq:key}
\abs{\IE h(X)-\IE h(Z)}
	\leq \norm{\Delta h}\binom{n}{2}\bbclr{4\bclr{1-\tsfrac{1}{2}\abs{\Phi}'(1)}}^{-1} \sum_{\ell=2}^L \abs{\beta_\ell} \sqrt{\Var(\Delta_{12} t_\ell(Z))}.
\end{eqnarray}
In particular if $h(x) = t(H,x) n^{-{|v(H)|}}$ is the density of appearances of graph $H$ in $x$, then $|| \Delta h || = O (n^{-2}),$
and  ${\Var(\Delta_{12} t_\ell(Z))}= O(n^{-1}).$
Thus for such functions $h$ the bound will tend to 0 with $n \rightarrow \infty;$ 
the statistics $t_\ell(x)$ are of this type. Also,  as $h(x)$ is bounded by ${\rm aut}(H)$, the number of automorphisms of $H$, we have for $g(x) = h(x) /{\rm aut}(H)$ that $0 \le g(x) \le 1$ and $g(x) = O(1)$ as well as  $\parallel \Delta (g^m) \parallel = O (n^{-2})$ for any 
$m>0$. Thus, {for independent realisations of} 
$X$ and {$Z$} on the same probability space, 
all moments of $T(X,Z) = g(X) - g(Z)$ converge {to} 0 as $n \rightarrow \infty$ and are uniformly bounded. 
{From the convergence of all moments} 
it follows that $T(X,Z)$ converges to 0 in probability and hence the difference between counts in the two network models converges to 0 in probability,
Thus, 
with the convention that $0/0=0$, $\frac{n_{\uk} (X) } {N_{\uk} ({X})} - \frac{n_{\uk} (Z) } {N_{\uk} (Z)} $ 
converges to 0 in probability as $n \rightarrow \infty$.

It remains to show that $\frac{n_{\uk} (Z) } {N_{\uk} (Z)} $
is a consistent estimator for $\astar$, {the edge probability of  the ER graph {$Z$}}. 
{To see this, we use} Proposition A.{2} in  {the supplementary information for} \cite{xu2022agrasstarxiv}, 
which gives that $\frac{n_{\uk} ({Z}) } {N_{\uk} ({Z})}$ {converges to 
$a^*$ {in probability as $n \rightarrow \infty$.} 
} 
\end{proof}

\subsection{Diversity guarantee} \label{subsec:diversity}

The next result shows that 
{\steingen} 
samples 
are {expected to be} well separated.

\begin{proposition}\label{prop:hamming}
    Under Assumption \ref{assum:er_approx}, the expected Hamming distance between two consecutive steps in the Glauber dynamics converges to $2 a^* ( 1- a^*).$
\end{proposition}

{\begin{proof} In the Glauber dynamics at each step at most one edge is flipped. The Hamming distance $d_H$ from \eqref{eq:hamm} between two consecutive instances is 1 when there is a flip, and otherwise, it is 0. 
Thus, {if $X(u)$ and $X(u+1)$ are two consecutive steps in the Glauber dynamics of}  $\ergm (\beta, t)$, and if $Z(u)$ and $ Z(u+1)$ are two consecutive steps of the ER($\astar$) Glauber dynamics, then by the triangle inequality
\begin{eqnarray*}
\lefteqn{ \mathbbm{E}    d_H(X (u), X(u+1))}\\
&\le&  \mathbbm{E} d_H(X (u), Z(u))  +
\mathbbm{E}  d_H(Z(u), Z(u+1)) + \mathbbm{E}  d_H(Z(u+1), X(u+1)) . 
\end{eqnarray*}
This inequality holds for any coupling between $X(u)$ and $Z(u)$, and for any coupling between $X(u+1)$ and $Z(u+1)$.
{In the Bernoulli random graph, edge indicators are independent, and thus, }with $S$ denoting the randomly chosen index from $[N]$,
\begin{eqnarray*}
\lefteqn{ 
\mathbbm{E}   d_H(Z(u), Z(u+1))  }\\ &=& \frac1N \sum_{s \in [N]} \mathbb{P} (Z(u)  \ne Z(u+1) | {S=s} 
)
    )
    \\&=& \frac1N \sum_{s \in [N]} \left\{ \mathbb{P} (Z (u){^{(s)}}  =1 ,  Z(u+1){^{(s)}}  =0 | {S=s} )
    +  \mathbb{P} (Z(u){^{(s)}}  =0 , Z_s(u+1)  =1 | {S=s} ) 
    \right\}\\
    &=& 2 \astar ( 1 - \astar).
\end{eqnarray*}
Moreover, from Remark 1.14 in \cite{reinert2019approximating} it follows that we can couple $X(u)$ and $Z(u)$ so  that there are on average $O(n^{3/2})$ edges that do not match. For this coupling, 
$ \mathbbm{E} d_H(X (u), Z(u))  = O(n^{-\frac12}),
$
and the same argument gives that we can couple $X(u+1)$ and $Z(u+1)$ such that
$   \mathbbm{E} d_H(X (u+1), Z(u+1))  = O(n^{-\frac12}).
$
Hence, as $n \rightarrow \infty$, the expected Hamming distance converges to the 
value $2 a^*(1-a^*)$. 
\end{proof}

We note that $2 a^* ( 1- a^*)$ is the expected Hamming distance {between two consecutive networks generated by}  the Glauber dynamics of an ER($\astar$) model. {Thus, the expected Hamming distance between two independent ER($\astar$) graphs $Y$ and $Z$ is
$\E d_H(Y,Z) = \sum_{s \in [N]} \mathbb{P} (Y^{(s)} \ne Z^{(s)}) = 2 N a^* (1 - a^*). $
{As} $  \E d_H(Y,Z)/N $ is independent of the number of vertices $n$,  in our experiments we 
scale the Hamming distance by $1/N$.}


\subsection{Mixing time considerations} \label{subsec:mixing}

{Although the distributions {of the generated graphs are close to that of the model generating the input graph $x$,}
the Glauber Markov process quickly `forgets' its starting point {$x$}. Indeed {Theorem 5 in  \citet{bhamidi2011mixing} gives  that under Assumption \ref{assum:er_approx},
the mixing time of the Glauber Markov chain is of order $\Omega(N\log N)$; we recall that the mixing time of a Markov chain is the number of steps needed in order to guarantee that the chain, starting from an arbitrary state, is within distance $e^{-1}$ from the stationary distribution. 
As in each step of the Glauber dynamic, a vertex pair is chosen independently with the same probability, the time until all $N$ possible vertex pairs have been sampled has the ``coupon collector
problem'' distribution, with mean $N \log N + \gamma N + \frac12 + O(N^{-1}$) 
(where $\gamma$ is the Euler-Mascheroni constant) and variance bounded by $\pi^2 N^2 / 6$.
the time of mixing, two chains started in different initial conditions will both be close to the stationary distribution. Hence, as stopping rule in the {\steingen} algorithm \Cref{alg:steingen} we suggest to use $r = \lfloor N \log N + \gamma N + \frac12 \rfloor$.} 
{For example when $n=50$ then $r=$ 9419.}

\subsection{Stability of {\steingen}} \label{subsec:stability}

To show the {stability of the network generation}  we  use Theorem 2.1 from \cite{reinert2019approximating}, as follows. 
Define the  $N\times N$ \emph{influence matrix} $\hat R$ for the Glauber dynamics of the distribution of $X$ by
$
\hat R_{r s}:=\max_{x\in\{0,1\}^N}\bbabs{q_X\bclr{\clr{x\s{s,1}}\s{r,1}|x\s{s,1}}-q_X\bclr{\clr{x\s{s,0}}\s{r,1}|x\s{s,0}}}.
$
Then $\hat R_{r s}$ is the maximum amount that the conditional distribution of the $r^{th}$ coordinate of $x$ can change due to a change in the $s^{th}$ coordinate of $x$. For $1\leq p\leq \infty$, let $\norm{\cdot}_p$ be the $p$-norm on $\IR^N$, and define the matrix operator $p$-norm
$
\norm{A}_p:=\sup_{v\not=0} \frac{\norm{A v}_p}{\norm{v}_p}.
$
{
\begin{Assumption} \label{ass:2}
 Assume that the distribution of $X$ is such that there is an $N\times N$ 
matrix $R$ satisfying that for all $r,s\in[N]$, and some $1\leq p 
\leq
\infty$ and $\eps=\eps_p>0$, we have 
$\hat R_{rs}\leq R_{rs}$ and $\norm{R}_p\leq 1-\eps<1.$
\end{Assumption}
If $X\sim ERGM(\beta, t)$  then \cite{reinert2019approximating} show that Assumption \ref{assum:er_approx} implies  Assumption \ref{ass:2}. However, for a stability result, we may be interested in  {comparing $X$ and $Y$ having possibly different distributions, such as} 
$X$ and $Y$ having the distribution of two consecutive steps in the Glauber dynamics. The general result is as follows.}

\begin{theorem}[Theorem 2.1 in \cite{reinert2019approximating}] \label{thm:key}
Let $X,Y\in\{0,1\}^N$ be random vectors, $h:\{0,1\}^N\to\IR$, and assume that the continuous time Glauber dynamics for the distribution of $X$ is irreducible {and satisfies Assumption \ref{ass:2}}.   For $s\in[N]$, set $c_s:=\norm{\Delta_s h}$ and 
$v_s(y):=\abs{q_X(y\s{s,1}|y)- q_Y(y\s{s,1}|y)}$, and $c:=(c_1,\ldots, c_N)$ and $v(Y):=(v_1(Y),\ldots, v_N(Y))$.  
Then for $q:=p/(p-1)$, {we have}
$
\babs{\IE h(X)-\IE h(Y)}  
\leq  \eps^{-1} \norm{c}_{q}  \, \IE \norm{v(Y)}_p.
$
\end{theorem}

{
Thus, if the conditional probabilities $q_X$ {satisfy \Cref{ass:2} and if the differences $v_s$ between $q_X$ and $q_Y$ are small}  then the networks which they generate {are close}, measured by expectations of test functions.  If the networks $X$ and $Y$ are from two consecutive steps of the Glauber dynamics, 
Proposition \ref{prop:consistent} shows that for large $n$ the corresponding estimated conditional probabilities will {indeed} be close in probability.}

\subsection{{Measuring sample fidelity via 
total variation} distance 
} \label{subsec:tv}

{In this paper we assess the goodness of fit of the generated data to the hypothesised model using {\gKSS} and {\AgraSSt}.}
{To assess fidelity {of graph samples} empirically,  we use the total variation distance $d_{TV}$ from \eqref{eq:dtv} between the empirical degree distributions of a synthetically generated network and the input network.} 
For two networks $X^{(i)}, i=1,2$ 
 the empirical probability mass function of their degrees is  
$G^{(i)} (k) = \frac{1}{n} \sum_{v=1}^n \mathbbm{1} (\deg^{(i)} (v) = k), \quad i=1,2, $ with $\deg^{(i)} (v)$  denoting the degree of vertex $v$ in $X^{(i)}, i=1,2$. {The total variation distance between these empirical distribution functions is} 
\begin{eqnarray*}
 d_{TV} (G^{(1)}, G^{(2)}) 
 &=& \frac12 \sum_{k=0}^{n-1} 
 \left| \frac{1}{n} \sum_{v=1}^n  \mathbbm{1} (\deg^{(1)} (v) = k)
 - \frac{1}{n} \sum_{v=1}^n \mathbbm{1} (\deg^{(2)}  (v) = k)\right| .
\end{eqnarray*}
{For a collection of $r$ generated networks} with $G^{(0)} $ the degree distribution in the observed network and $G^{(i)} $, $i=1,2$, the degree distribution in the $i^{th}$ simulated network, as in \cite{xu2021stein} we  {measure fidelity by}
 the average empirical total variation distance 
 \begin{equation} \label{eq:empdtv}
 \frac{1}{{r}} \sum_{i=1}^{{r}} d_{TV} (G^{(0)}, G^{(i)}).\end{equation}
{To interpret this measure we note} 
 that even if two networks {are independently generated} from the same  distribution, their empirical degree distributions $G^{(1)}$ and $G^{(2)}$ may not completely agree.  Assume that $\P (G^{(1)} (k) =  G^{(2)} (k)) <1$; this is the case for example in Bernoulli random graphs with edge probability $0 < p < 1$.}  While $\mathbbm{E} G^{(1)} (k) = \mathbbm{E} G^{(2)} (k)$ for all $k$, the {expectation of the } empirical total variation distance does not  vanish.
To see this, 
 as  $G^{(1)} (k)$ and $G^{(2)} (k)$ are exchangeable {if they are generated from the same distribution}, 
\begin{eqnarray*}
{\mathbbm{E}  d_{TV} (G^{(1)}, G^{(2)}) }
 &=& \frac12 \sum_{k=0}^{n-1} 
 \mathbbm{E} \left| G^{(1)} (k) - G^{(2)}  (k)\right| \\
&=&  \sum_{k=0}^{n-1} 
 \mathbbm{E}  ( G^{(1)} (k) - G^{(2)}  (k) ) \mathbbm{1} (  G^{(1)} (k) >  G^{(2)} (k) ) 
 \end{eqnarray*} 
{by symmetry}. As $ \mathbbm{E}  ( G^{(1)} (k) ) = \mathbbm{E} ( G^{(2)}  (k) ) $ and as 
 $\P (  G^{(1)} (k) \ne G^{(2)}  (k) ) > 0$
 it follows that 
 ${\mathbbm{E}  d_{TV} (G^{(1)}, G^{(2)}) }>0$
 so that even if the distributions were identical, the average empirical total variation distance  would not vanish. 
 
 {When the underlying network model is a Bernoulli random graph, ER$(p)$,  the degree of a randomly picked vertex is binomially distributed with parameters $n-1$ and $p$.
 However,} 
 due to the dependence in the degrees, the random variables $G^{(1)}$ and $G^{(2)}$ are not quite binomially distributed. Using the binomial approximation from \cite{soon1996binomial} with the coupling from \cite{goldstein1996multivariate}, {we can approximate the degree distribution by the distribution of a collection of independent binomially distributed random variables $D_k \sim \text{Binomial}(n-1, p_k)$ with 
 $p_k = \mathbbm{P}(\deg^{(1)} (v) = k) = {{n-1} \choose k} p^k (1-p)^{n-1-k} $, for $k=0, \ldots, n-1$. 
 {Then} $G^{(i)}(k) \approx \frac1n D^{(i)}_k$ where $D^{(i)}_k \sim \text{Binomial}(n-1, p_k)$ are 
 independent, {and} 
 \begin{eqnarray*}
 {\mathbbm{E}  ( | G^{(1)} (k) - G^{(2)}  (k) | }
 &\approx&  \frac1n \mathbbm{E}  
 | D^{(1)}_k - D^{(2)}_k|  = \frac1n \mathbbm{E}  \{ (D^{(1)}_k + D^{(2)}_k -  2 \min (D^{(1)}_k , D^{(2)}_k)\}.
 \end{eqnarray*}
In \cite{craig1962mean} it is shown that 
$ \frac12  \sqrt{\frac{n}{\pi}} \le \min (D^{(1)}_k , D^{(2)}_k) \le np + 2 p (1-p) \sqrt{\frac{n}{\pi}}, 
$
and hence
\begin{eqnarray}
     \frac1n \mathbbm{E}  
 | D^{(1)}_k - D^{(2)}_k| \in  \left[4 p (1-p) \sqrt{\frac{1}{n \pi}} , \sqrt{\frac{1}{n \pi}}  \right]. \label{eq:dtvbound}
\end{eqnarray}
We use the upper bound $\sqrt{\frac{1}{n \pi}}$ as guideline.
}

 %
{As} the underlying network generation method is unlikely to be $G(n,p)$, 
we {give the bound} 
here for heuristic consideration only. 
{In} our synthetic experiments, we simulate the empirical total variation distance between the degree distributions under the null hypothesis.}

\section{Experimental results}\label{sec:exp}

{In our experiments, we assess} 
the two {\steingen} {generators from}
\Cref{sec:method}: 
\textbf{SteinGen\textunderscore{nr}} uses a fixed 
${q({s,1} | {\Delta_s t(x)}  )}, s \in [N],$
estimated from the input graph; 
\textbf{SteinGen} re-estimates ${q({s,1} | {\Delta_s t(x)}  )}$
using the generated graph samples.
We 
{compare the {\steingen} generators against two types of} 
graph generation methods.
The first type estimates the parameters $\beta$ in \eqref{eq:ergm} and then uses MCMC to generate samples from the estimated distribution. {Here we use for parameter estimation}
\textbf{MLE}, the maximum likelihood estimator based on {an} MCMC approximation \citep{snijders2002markov};
\textbf{MPLE}, 
{a} maximum pseudo-likelihood estimator, {see} \citet{schmid2017exponential}, and \textbf{CD}, {an estimator based on the} contrastive divergence approach \citep{asuncion2010learning}. 
{Our} implementation 
{uses the} \texttt{sna} suite \citep{butts2008social} and {the} \texttt{ergm} package  \citep{ergm} in R. The second type of graph generation method 
{is implicit. Here we explore the implicit graph generators \textbf{CELL} \citep{pmlr-v119-rendsburg20a} and \textbf{NetGAN} \citep{bojchevski2018netgan}.} 
\textbf{CELL} 
{is} a  
cross-entropy low-rank logit approximation 
that learns 
underlying random walks for the graph generation\footnote{{The CELL implementation is adapted from the code at  \url{https://github.com/hheidrich/CELL}}.}. 
{\textbf{NetGAN} is a}
graph generative adversarial network method.  
Both \textbf{CELL} and \textbf{NetGAN} can learn and generate graphs from a single observation.

\subsection{{Measuring fidelity and diversity}} To assess sample quality in terms of fidelity to the {distribution generating the} input network,  we report various network statistics for the generated networks. Moreover, for networks generated from synthetic models, we report rejection rates of a gKSS test {as described in \Cref{sec:gkss}}; for real-world networks, we use an AgraSSt test {as described in \Cref{sec:agrasst}}.
{As kernel}  
we use 
{a} Weisfeiler-Lehman (WL) graph kernel  \citep{shervashidze2011weisfeiler} {with level parameter 3, because WL graph kernels}  
have been shown to be effective for graph assessment problems \citep{weckbecker2022rkhs, xu2021stein}.
{The gKSS test uses a Monte-Carlo based test threshold which in the synthetic experiments is determined using $200$ samples generated from the true generating model.} {When} AgraSSt and {\steingen}\textunderscore{nr} {use the same network statistics, as both} estimate the conditional probability once, 
we expect the rejection rate of {\steingen}\textunderscore{nr} {to be} 
close to the test level.
We report the proportion of rejected gKSS {or {\AgraSSt}}  tests for a test at level $\alpha = 0.05$; we aim for a proportion of rejected tests being close to this level.  

{We also assess the fidelity of individual samples. For a 
{sample} of ${m}$ generated 
{networks} $x(1), \ldots, x({m})$  and $x(0)$ the initial network, we use}
as sample-based measure for fidelity the average empirical total variation distance \eqref{eq:empdtv} between the empirical degree distributions of the generated network {samples} and the input network, a measure which is also employed in \cite{xu2021stein} and motivated by the graphical test in \cite{hunter2008goodness}, {see \Cref{subsec:tv}}.
 To assess sample diversity, {for a trial $i$ we first generate a network $x(0,i)$ which we then use as input network for generating a sample $x(1,i), \ldots, x(m,i)$ of size $m$}. We report the {scaled average Hamming distance ${\bar{d}}_H(i):=\frac1{{m}N} \sum_{j=1}^{m} d_H(x(j,i), x(0,i))$} between 
 {the} generated sample{s} and the input network. 
 {Here we divide by $N$, the maximal Hamming distance on networks with $n$ vertices and $N$ potential edges, to keep the measure bounded between 0 and 1; see \Cref{subsec:diversity}.} 
 {If we run $w$ trials then we report the average ${\bar{d_H}} =
 \frac{1}{w} \sum_{i=1}^w {\bar{d}}_H(i)$
 where ${\bar{d}}_H(i)$ is the average Hamming distance in trial $i$. To indicate variability we also report the average standard deviation 
 $sd:= \frac{1}{w} \sum_{i=1}^w sd (d_H(i) ),
 $
 where $sd(d_H(i))
 = \left( \frac1m \sum_{j=1}^m (d_H ( x(j,i), x(0,i)) - {\bar{d}}_H(i))^2\right)^{1/2}$
 is the standard deviation of the Hamming distance in trial $i$.}
 The variability of the Hamming distance is used to illustrate the variability in the generated samples.


{To visualise the fidelity-diversity trade-off w}e plot (1 - average empirical total variation distance  of the degree distributions), {abbreviated $1 - {TV}$ Distance,} against the {average} Hamming distance. {For {$1 - {TV}$ Distance} again we average over $w$ trials.} The closer to the top-right corner, the more 
fidelitous and diverse are the generated samples. In the interpretation of the plots, we take note of the theoretical bounds from 
\Cref{subsec:tv}. 


\subsection{Synthetic network
simulations
}
\label{sec:expsim}
{In our} synthetic experiments, input networks 
are generated under four different ERGMs.
{With $E(x)$  the number of edges,  $S_2(x)$} the number of 2Stars, and $T(x)$ the number of triangles in a network $x$, we generate networks {on $n=50$ vertices}  
from 
\begin{enumerate}
\item  an Edge-2Star (E2S) model \citep{mukherjee2013statistics}, with  unnormalised density 
$
q(x) \propto \exp\{\beta_1 E(x) + \beta_2 S_2(x)\};
$ 
\item an Edge-Triangle (ET) model \citep{yin2016asymptotic} with unnormalised density
$
q(x) \propto \exp\{\beta_1 E(x) + \beta_2 T(x)\};
$
\item an Edge-2Star-Triangle (E2ST) model \citep{yang2018goodness,xu2021stein} with  unnormalised density
$
q(x) \propto \exp\{\beta_1 E(x) + \beta_2 S_2(x) + \beta_3 T(x)\};
$
and 
\item an  ER($\beta_1)$ model.
\end{enumerate}
We choose $\beta_1=-2$, $\beta_2=\frac{1}{n}$, $\beta_3=-\frac{1}{n}$. 
These models satisfy the  fast mixing condition in \citet{bhamidi2011mixing} {and Assumption \ref{assum:er_approx},} and are unimodal. 
{In this example, the choice of 
$r = N \log N + \gamma N + 0.5$ gives $r=9419$ as number of {steps}.
}

\paragraph{{Sample quality via gKSS}}
We first compare the quality of generated samples using {the rejection rate of gKSS tests at test level $0.05$}. 
For an input graph generated from each ERGM, we generate ${m=}30$ samples from each graph-generating method as one trial. We run $w=50$ trials 
The average gKSS value {over these $50$ trials is} shown in \Cref{tab:gkss_results}. 
In addition,  we  
report the rejection rate for the $50$ ``observed'' input samples  
as a baseline. 
The gKSS rejection rates 
which are closest to the true level $0.05$ are coloured red and the second-closest are coloured blue.

\begin{wraptable}{r}{0.61\textwidth}
    \vspace{-1.cm}
    \centering
    \caption{Rejection rates of gKSS tests {for $50$ trials,  with 30} 
    generated samples {for each trial, at} test level $\alpha=0.05$; {the} closest value to {the} test level {is} in red, and the second closest is in blue.}
   \vspace{-0.3cm}
    \label{tab:gkss_results}
    \begin{tabular}{l|rrrr}
\toprule
{Model} &         E2S &         ET &         E2ST &         ER \\
\midrule
MPLE &  0.393 &  0.133 &  0.370 &  {\color{blue}0.040} \\
CD &  0.413 &  0.200 &  0.403 &  0.030 \\
MLE &  0.253 &  0.127 &  0.250 &  0.036 \\
CELL &  0.080 &  0.100 &  0.190 &  0.020 \\
{NetGAN} &  0.110 &  0.160 &  0.280 &  0.086 \\
{\steingen}\textunderscore{nr} &  {\color{blue}0.021} &  {\color{blue}0.075} &  {\color{blue}0.105} &  {\color{red}0.050} \\
{\steingen} &  {\color{red}0.030} &  {\color{red}0.040} &  {\color{red}0.100} &  0.035 \\
\hline
Observed &  0.040 &  0.050 &  0.080 &  0.030 \\
\bottomrule
\end{tabular}   
\end{wraptable}

 \Cref{tab:gkss_results} shows that for E2S, ET and E2ST, {the}
parameter estimation methods have much higher gKSS {rejection rates than the other methods}. {The best performance is achieved by {\steingen}, followed by its faster variant {\steingen}\textunderscore{nr}. For the ER model, all generation methods achieve reasonable rejection rates, with  {\steingen}\textunderscore{nr} being completely on target in our simulation and MPLE not far behind.}  

\bigskip
\paragraph{Sample diversity via Hamming distance}

{If all generated samples are near-identical to the input network then the synthetic data may be of limited value. To assess variability,} 
%
Figure \ref{fig:hamming} {shows the average Hamming distance {between each generated sample and the input network} for {samples from the different methods for the above ERGMs (excluding ER),} 
plus/minus one standard deviation (sd), with the $x$-axis indicating the number {$r$} of {steps} 
generated {for {\steingen} and {\steingen}\textunderscore{nr}.  
The other methods do not generate consecutive samples and their Hamming distances are hence drawn as straight lines.  {For comparison we also give the theoretical bound $2 a^* ( 1-a^*)$ on the Hamming distance from Proposition \ref{prop:hamming}.}

\begin{figure}[t]
    \centering
    \begin{center}
    {
{\includegraphics[width=0.95\textwidth]{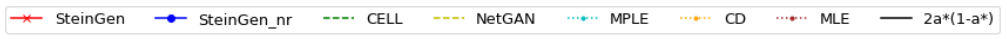}}}
    \subfigure[E2S Model]{
    \includegraphics[width=0.31\textwidth]{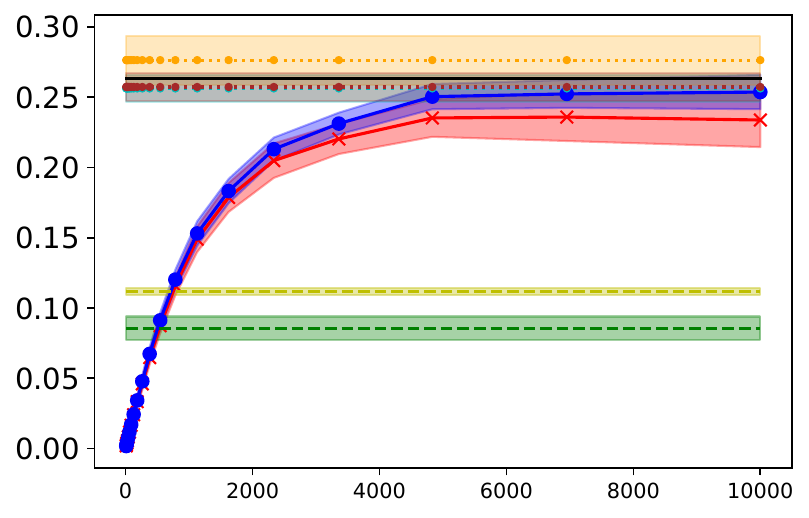}}
    \subfigure[E2ST Model]{
    \includegraphics[width=0.31\textwidth]{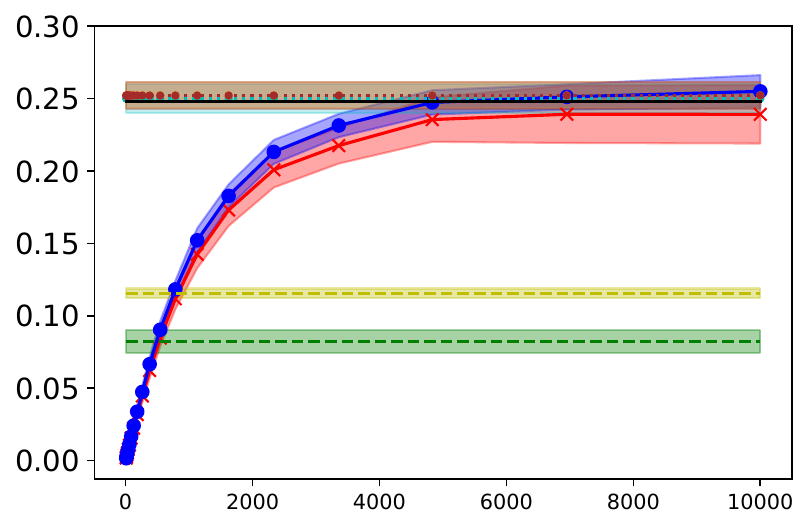}}
    \subfigure[ET Model]{
    \includegraphics[width=0.31\textwidth]{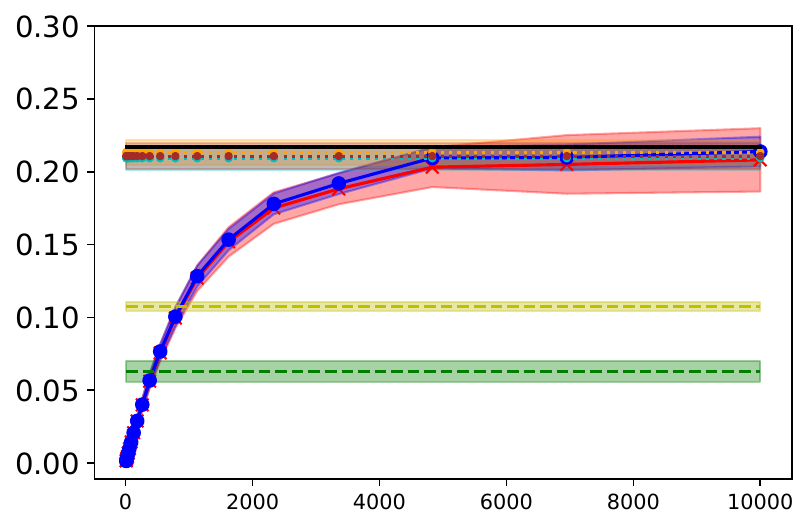}}
    \vspace{-0.1cm}
    \caption{Hamming distance between generated samples and {the} {initial {synthetic}  network} 
    {for a networks on ${n=}50$ vertices}; 
    {average} and standard deviation 
    of {$m=$}50 trials. 
    {In E2S, $
    2a^\ast(1-a^\ast)=0.263$;
    in E2ST, $
    2a^\ast(1-a^\ast)=0.248$;
    in ET, $
    2a^\ast(1-a^\ast)=0.217$.}
    \label{fig:hamming}}
    \end{center}
    \vspace{-1.5cm}
\end{figure}

From Figure \ref{fig:hamming}, we see that the parameter estimation methods have largest Hamming distance from the input network. As these methods use the input network only for parameter estimation and then generate networks at random, this finding is perhaps not surprising. However,
{\steingen} samples have much higher Hamming distance compared to those from CELL, indicating higher sample diversity. 
With the number of {steps in {\steingen},}
 the Hamming distance for both {\steingen}\textunderscore{nr} and {\steingen} samples increases and then stabilises and approaches the theoretical limit ${2}a^* ( 1 - a^*) $ from \Cref{prop:hamming}; this stabilisation provides another  natural criterion for {the number of steps for which to run} 
{\steingen} and {\steingen}\textunderscore{nr}. 
{While in \Cref{subsec:mixing} the theoretical underpinning gives $N \log N + \gamma N + 0.5=9419$ as guideline for the number of steps, in Figure \ref{fig:hamming} the results are already close to stable for step sizes of around 
$r=4000$, less than half of $N \log N + \gamma N + 0.5$.}
The variance of 
the Hamming distance for {\steingen} after stabilisation is higher than that of {\steingen}\textunderscore{nr}, indicating that the re-estimation procedure increases sample diversity.
The sample diversity achieved by {\steingen} and {\steingen}\textunderscore{nr} is close to that achieved by the parameter estimation methods.

{\paragraph{Fidelity-diversity trade-off} 
\Cref{fig:frontier}
shows the trade-off between fidelity and diversity in the simulated networks. 
The dotted red line shows the estimated {$1 - {TV}$ Distance} 
using 
$50$ simulated networks under the null model.
{As expected, a}s the number of {steps} $r$ increase{s}, the Hamming distances (diversity) for {\steingen} samples increase{s} while  $1 - TV$ Distance (fidelity) decrease{s}. 
However, 
the sample fidelity decreases only a by small amount and approaches the empirical total variation distance (the red dashed line). Compared to CELL and NetGAN, {\steingen} with large $r$ produces samples with simultaneously higher diversity and higher fidelity. 
The  bound $(\sqrt{n \pi})^{-1}$ {from \eqref{eq:dtvbound}} 
is not too far off. 

More synthetic experiments can be found in \Cref{app:synth}, {including 
different re-estimation intervals for updating the estimates of $\hat{q}$ (\Cref{app:reest}),  examples with multiple graph observations (\Cref{app:mult}), and {improving} the sample quality by selecting {samples} with the smallest {\gKSS} value  (\Cref{app:select}).} 
{The choice of graph kernels in {\gKSS} is explored in detail in \cite{weckbecker2022rkhs}.}


\begin{figure}[t!]
    \centering
    {
{\includegraphics[width=0.68\textwidth]{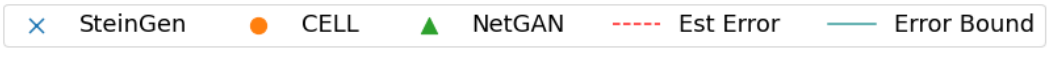}}}
    \subfigure[E2S Model]{
    \includegraphics[width=0.318\textwidth]{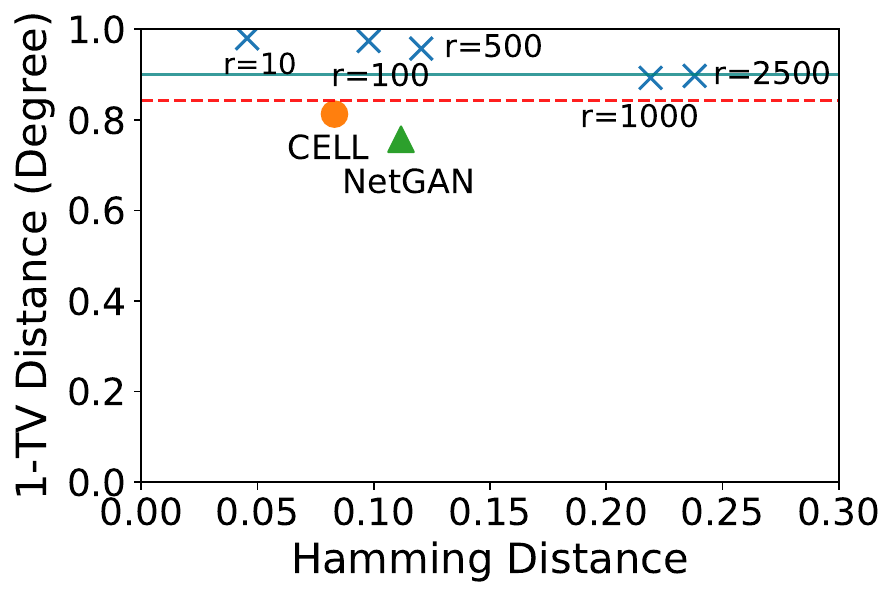}}
    \subfigure[E2ST Model]{
    \includegraphics[width=0.318\textwidth]{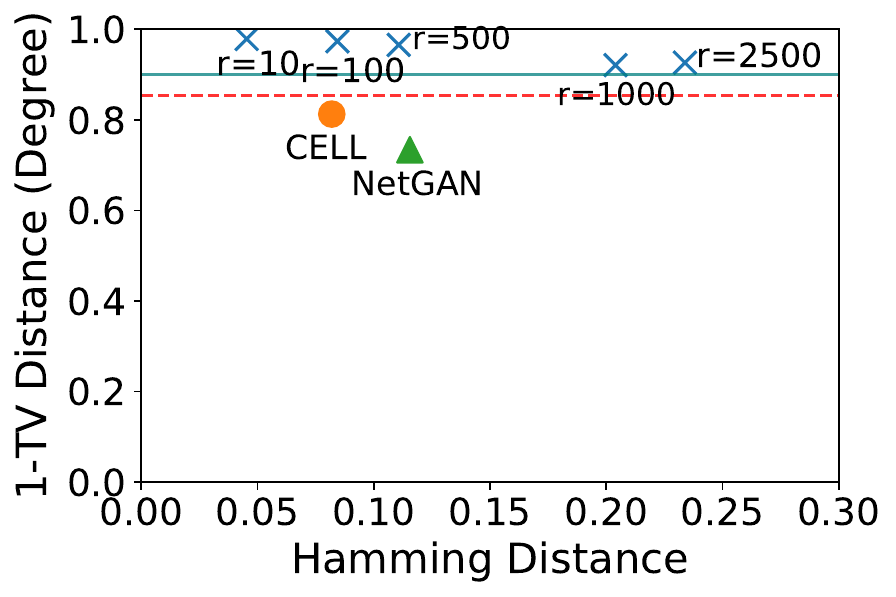}}
    \subfigure[ET Model]{
    \includegraphics[width=0.318\textwidth]{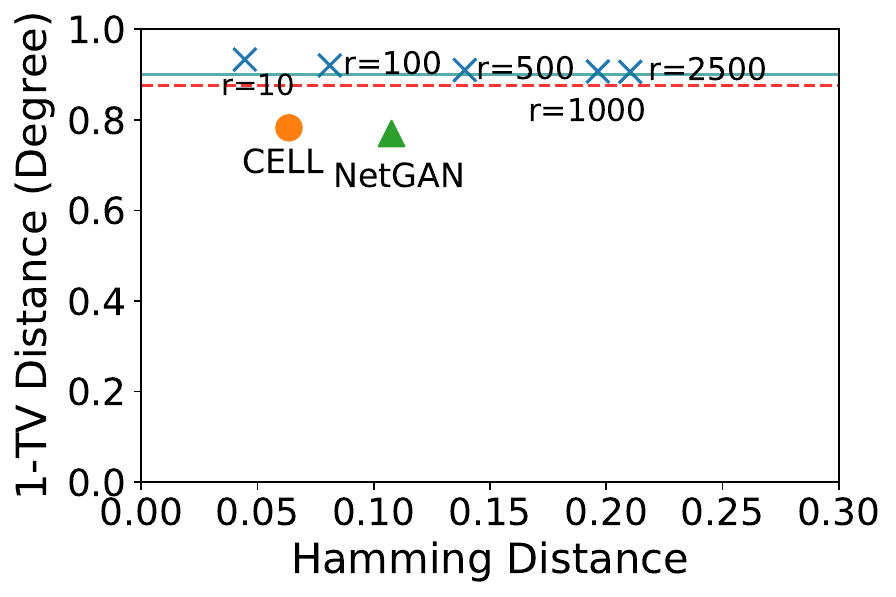}}
    \vspace{-0.5cm}
    \caption{Hamming distance versus TV distance {of} degree using generated samples; $r$ is the number of {steps} in {\steingen}.
    {The estimated error is estimated from simulations, the error bound is the bound $(\sqrt{\pi n})^{-1}$ from \eqref{eq:dtvbound}.
\label{fig:frontier}}
   }
\vspace{-1.12
cm}
\end{figure}

\paragraph{Runtime comparison}
{The time, in seconds,}
for generating a sample from each method for {our} E2ST model, {are, in order of speed: {\steingen}\textunderscore{nr} (0.0244), CELL (0.0487), {NetGAN (0.5265)}, {\steingen} (0.0559), MPLE (0.0929), and MLE (0.5090).}
{
{\steingen}\textunderscore{nr} is  the fastest method, but it has a less accurate gKSS test rejection rate than {\steingen}.}



\WFclear
\subsection{Real network applications}
\label{subsec:real}

As a real network example {we} use a 
teenager friendship network with 50 vertices described in \citet{steglich2006applying}; 
\cite{xu2021stein} propose an E2ST ERGM. {
\Cref{tab:teenager} shows some of its network summary statistics.} 

We generate $50$ samples from the input graph and compute  
the sufficient statistics Edge Density, Number of 2Stars and Number of Triangles for the generated samples from each method; their averages and standard deviations are shown in \Cref{tab:teenager}. 
The reported SteinGen values use $r=600$ {steps}.

\begin{table}[t!]
    \centering
    \caption{Teenager friendship network. {Closest to observed is in red, second-closest in blue.}
    \label{tab:teenager}}
        \vspace{-0.15cm}
    \begin{tabular}{c|lll|c|l}\toprule
 {} &  Density &     2Stars &  Triangles &     AgraSSt & Hamming \\
 \midrule 
 MPLE             &  0.0421 (2.42e-2) &   329 (80.4) &    75.52 (43.4) & 
 0.68
 &0.106 (2.22e-2)\\
 CD               &  0.2900 (1.10e-2) &  4537 (538) &   4146 (668) &  
 0.92
 & {\color{red}0.211} (1.03e-2)\\
 CELL             &  {\color{red}{0.0450}}  (3.46e-4) &    220 (14.1) &      22.50 (7.73) & 
 0.12
 & 0.0423 (3.32e-3)\\
{NetGAN}             &  {{0.1120}}  (1.38e-6) &    227 (13.3) &      9.28 (2.53) & 
 0.34
 & 0.0820  (5.07e-3)\\
{\steingen}\textunderscore{nr}   &  0.0516  (1.02e-3) &   {\color{blue}{362}} (14.9) &    {\color{blue}{88.90}} (24.8) &   
 {\color{red}0.06}
 & 0.0912 (9.95e-3)\\
 {\steingen}  &  {\color{blue}{0.0445}} (9.49e-4) &   {\color{red}{364}} (84.1) &     {\color{red}{85.75}}   (10.7) & 
 {\color{blue}0.08}
 & {\color{blue}0.107} (1.32e-2)\\
 \hline
 Teenager         &  0.0458 &    368 &      86.00 & pval=0.64 &  \\
    \bottomrule
    \end{tabular}
    \vspace{-0.5cm}
\end{table}

{For this network, the} MCMCMLE {estimation procedure in \tt{ergm}} does not converge.
CELL captures the edge density and 2-Star statistics well, {but not} 
the triangle counts. CD has the highest variability but does not capture the sufficient network statistics. MPLE estimates the sufficient statistics reasonably well but is outperformed by  {\steingen}. 
{\steingen}\textunderscore{nr} also performs  well in capturing the sufficient statistics.
{As the true model for the teenager network is unknown and hence the gKSS test does not apply, in \Cref{tab:teenager} we also report the proportion of rejections of the 
kernel-based 
Approximate graph Stein Statistic (AgraSSt) goodness-of-fit test {\eqref{eq:agrassteq}} 
to assess the sample quality{; see \Cref{sec:agrasst} for details}. 
This test {uses an approximate model which estimates 
the conditional  probabilities in \Cref{eq:cond_prob}, given the number of edges, 2stars, and triangles, from the observed Teenager network, without an explicit underlying ERGM.}

\begin{wrapfigure}{r}{.48\textwidth}
\centering
\vspace{-0.7cm}
\includegraphics[width=.45\textwidth]{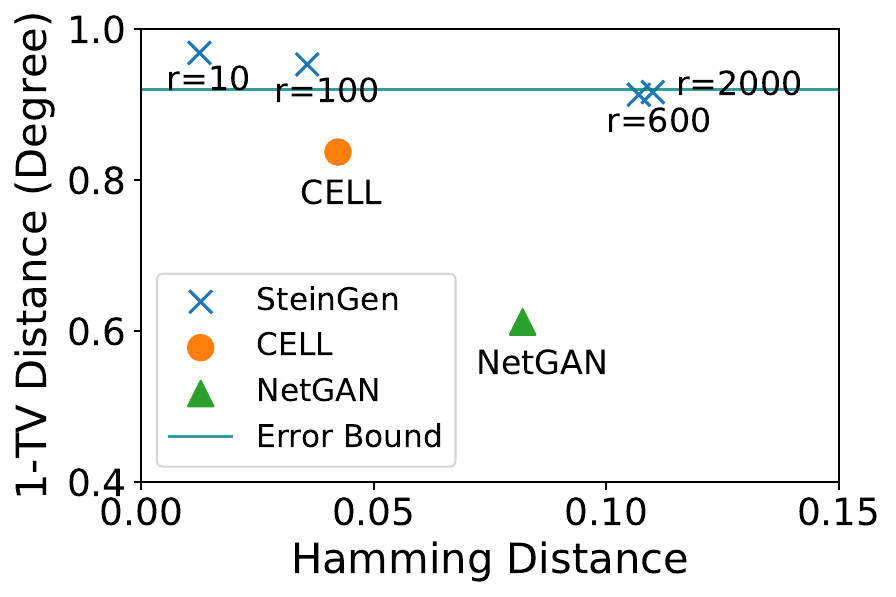}
\vspace{-.5cm}
\caption{Hamming distance versus $1-TV$ Distance {of} degree for the teenager network; $r$ is the number of {steps}
in {\steingen},  {the blue line is the error bound $(n \pi)^{-\frac12}$ from \eqref{eq:dtvbound}.}\label{fig:teenager_frontier}}
\vspace{-.5cm}
\end{wrapfigure}

We generate $100$ samples from each method and perform the AgraSSt test using $200$ samples generated from the approximate model to determine the rejection threshold at test level $\alpha=0.05$. We {also} report the AgraSSt test $p$-value of the Teenager network; {the value indicates that the observed network can plausibly be viewed as having the estimated conditional distribution.}

\Cref{tab:teenager} shows that 
{\steingen} has {the} rejection rate {which is closest} to {the} test level 0.05, {followed by {\steingen}\textunderscore{nr} and CELL,} while 
the parameter estimation methods MPLE and CD have a much higher rejection rate.
{Regarding diversity,} the Hamming distance from CELL is the lowest, indicating that perhaps the generated samples are very similar to the original 
Teenager network. CD produces the largest Hamming distances on average, but 
the sample quality is low. The next highest diversity is produced by {\steingen}.

Moreover, \Cref{tab:network_stats}  shows some {additional} 
{standard} network statistics {to match those used in \cite{pmlr-v119-rendsburg20a}, except the power law exponent which is not often informative for small networks}: 
shortest path (SP), largest connected components (LCC), assortativity (Assortat.), clustering coefficient (Clust.) and Maximum degree of the network. 
{{\steingen} performs the best on 
first four 
of these statistics, within one standard deviation of the observed values, closely followed by {\steingen}\textunderscore{nr}, while the other methods show statistically significant deviation from the observed statistics.}
{However for the maximum degree,} 
NetGAN performs 
best, although closely followed by {\steingen}.}




Moreover, we plot the Hamming distance versus $1-TV$ Distance to visualise the trade-off between the quality of generated samples and the diversity. \Cref{fig:teenager_frontier} shows that with {increasing} number of {\steingen} {steps}, the increase in total variation  distance is small compared to the gain in Hamming distance.
{\steingen} produces samples with higher 
{fidelity} and, {for} $r=600$ or $r=2000$, {also with} larger diversity 
{than} CELL and NetGAN. 

{As an aside, for a network of size $n$ the crude upper bound on the total variation distance of $(n \pi)^{-\frac12}$ derived in \eqref{eq:dtvbound} gives  0.9202115; {\steingen} samples are not far off.}
{The effect of different kernel choices for the teenager network is explored in \Cref{subsec:teenager}, \Cref{tab:teenager_kernel}. While there are numerical differences, the {\AgraSSt} rejection rates are qualitatively similar for the different kernels.}


\begin{table}[t!]
\small
    \centering    
    \caption{Additional network statistics.}\label{tab:network_stats}
            \vspace{0.25cm}
    \begin{tabular}{l|lllll}
    \toprule
       {}  &    SP &    LCC &  Assortat. & Clust. & Max(deg)
       \\
       \midrule
       Teenager         &  3.39 &  33.00 &       0.172 & 0.9056 & 5.00 \\
       \hline {\steingen}  &  {\color{red}{3.49
}}(0.299) & {\color{blue}{33.20}} (0.678) &     {\color{red}{0.163 }}(0.031) & {\color{red}{1.027}} (0.018) & 11.50 (1.688) \\
{\steingen}\textunderscore{nr} &{\color{blue}{3.76 }}(0.464) &   {{35.25}} (2.66) &     {\color{blue}{0.144}} (0.056) & {\color{blue}{1.226 }}(0.098) & 13.00 (1.712)\\
CELL & 5.38 (0.905) &  45.30 (5.56) &       0.103 (0.089) & 0.191 (0.025) & {\color{blue}{11.20}} (0.980) \\
{NetGAN} & 2.66 (0.034) &  {\color{red}{33.00}} (0.) &       0.098 (0.085) & 0.132 (0.035) & {\color{red}{9.333}}  (1.014)\\
MPLE &2.048 (1.17) &  20.05 (8.48) &       0.765 (0.143) & 0.187 (0.024) & 11.30 (1.269) \\
CD & 1.148 (0.224) &  25.00 (0.100) &       0.985 (0.020) & 0.190 (0.027)& 12.00 (1.673)\\
    \bottomrule
    \end{tabular}
\vspace{-.5cm}
\end{table}



\section{Conclusion and discussions}\label{sec:conclusion}
{{\steingen} is a synthetic network generation method which is based on Stein's method and can be used even when only one input network is available; no training samples are required. 
{\gr{I}n our experiments} {\steingen} achieves a good balance between a high sample quality as well as a good sample diversity.
In  \Cref{app:real}} {of the Appendix} we include additional 
{experiments on} real network experiments, {namely} 
the Florentine marriage network from  \cite{padgett1993robust},  and two protein interaction networks, for EBV \citep{hara2022comprehensive} {and for yeast \citep{von2002comparative}. In these experiments the general pattern is confirmed, but with sometimes different orderings of CELL and NetGAN.}
With only one observed network, {we find that} {\steingen} generates synthetic samples which are close in distribution to the observed network while being dissimilar from it. Moreover, {\steingen} comes with theoretical guarantees.

{\steingen} outperforms its competitors partly through its re-estimation step which implicitly captures variability in the distribution from which the observed network is generated.   We also propose a faster method, {\steingen}\textunderscore{nr}, which avoids the re-estimation step and also performs well. As an intermediary method, 
{one could} re-estimate 
${\widehat q(s,1 | {\Delta_s t(x)}  )}$ in Algorithm \ref{alg:steingen} 
only after a fixed number of samples have been generated, see {{\Cref{app:reest}
for details {and} additional experimental results.}} 

{
While 
here 
{\steingen} is presented in the setting of ERGMs, using the \AgraSSt\,  approach {from} \cite{xu2022agrasst}, it is straightforward to generalise the approach to networks for which the underlying distribution family is unknown. 
Moreover, {\steingen} can easily be expanded to take multiple graphs as input, gaining strength in conditional probability estimation; for details see \Cref{app:mult} in the Appendix. One could also generate multiple samples and use {\AgraSSt} to select the best samples for the next network generation step, see \Cref{app:select} in the Appendix; this could be viewed as related to particle filtering. 
{Moreover, learning 
suitable network statistics $t(x)$ could be an interesting future research direction.}

{{\steingen} has some shortcomings: It disregards any attributes on the network. It does not naturally apply to time series of networks. It also does not come with any privacy preserving guarantees. Extending {\steingen} to these settings will be part of future work.}

Finally a note of caution: when applying {\steingen}, ethical aspects should be taken into consideration. One could think of situations in which synthetic networks distort the sensitive narrative of the data. Moreover, if the synthetic networks are used for crucial decision making such as in healthcare, extra care is advised.}

{\small
\acks{{The authors would like to thank Chris Oates (Newcastle) for a very helpful discussion.} 
G.R. and W.X. acknowledge the support from EPSRC grant EP/T018445/1.
G.R is also supported in part by EPSRC grants EP/W037211/1, EP/V056883/1, and EP/R018472/1.}
}
WX is also supported by Deutsche Forschungsgemeinschaft (DFG, German Research Foundation) under Germany’s Excellence Strategy – EXC number 2064/1 – Project number 390727645.


\bibliography{main}







\appendix


\section{More on parameter estimation methods}
\label{app:parameter}

As shown in \Cref{sec:exp}, parameter estimation methods based on MPLE, CD and MLE {can achieve high sample diversity but in our experiment they usually show low sample fidelity.} 
 {To better understand this behaviour here 
we} estimate {the} parameters {in the three models E2S, E2ST and ET {in the same setup as in \Cref{sec:expsim}, with $n=50$ vertices and the same true parameter 
{values} $\beta_1 = -2, \beta_2 = \frac1n, \beta_3 = -\frac1n$.  We note that parameter estimation methods estimate the parameters jointly for maximising the likelihood; for an observed network $x$ the linear combination $\sum_{l=1}^L \beta_l t_l(x)$ is the basis of the estimation. Hence we would not expect to see unique parameter estimates, but we would expect a linear combination of them to stay approximately constant.}}

\Cref{fig:param_est} {shows the results for the estimates of $\beta_1$ versus $\beta_2$; the true parameter combination is  indicated by a magenta star. In} 
the E2S {model},
which is arguably the easiest of the three models considered, 
{in} all three methods {there is an approximately linear relationship between the} estimates {of} $\beta_1$ and $\beta_2$,
relating to 
maintaining a similar density of the generated graphs. However, in {the} E2ST and ET {models}, the parameter estimation can be very inaccurate, for all three methods. {Hence,} 
a higher rejection rate in a gKSS test for the parameter estimation methods, {as observed in Table \ref{tab:gkss_results} in the main text,} is not unexpected.

\begin{figure}[tp!]
    \centering
    \subfigure[E2S Model]{
    \includegraphics[width=0.31\textwidth]{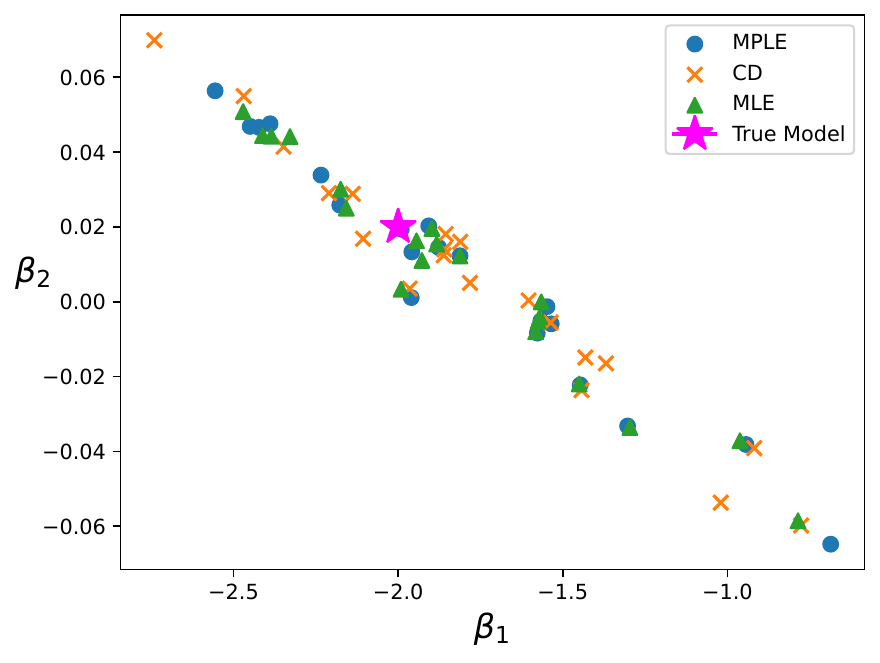}}
    \subfigure[E2ST Model]{
    \includegraphics[width=0.315\textwidth]{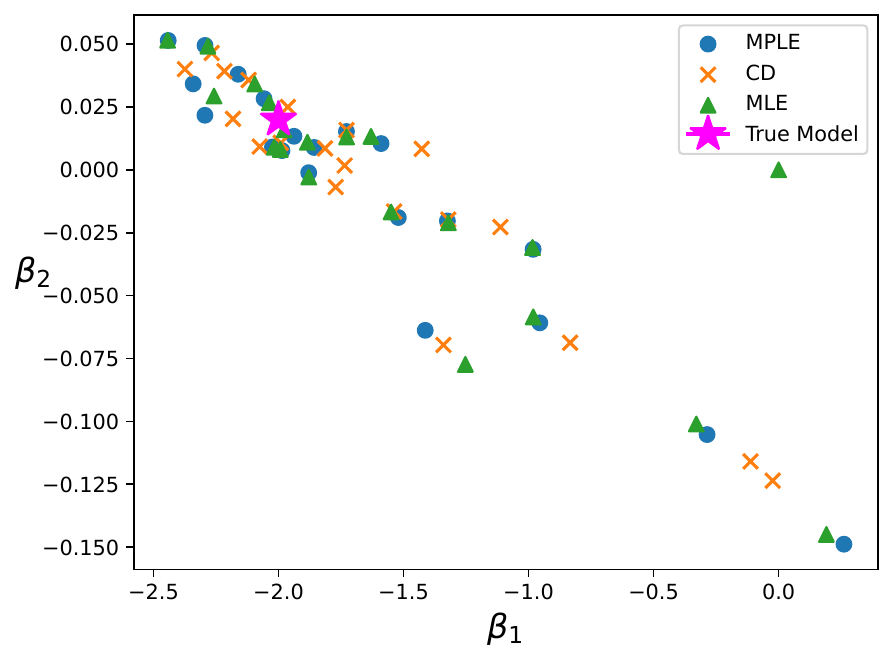}}
    \subfigure[ET Model]{
    \includegraphics[width=0.31\textwidth]{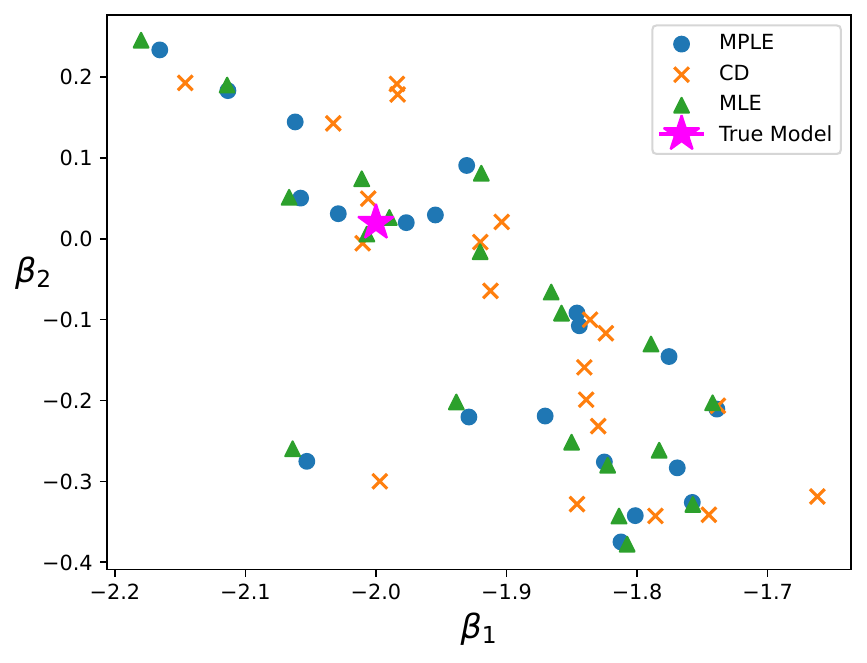}}
    \vspace{-0.1cm}
    \caption{Estimated parameters $\beta_1$, $\beta_2$ for $n=50$. {The plot shows considerable variability in the parameter estimation.}\label{fig:param_est}
    }
    \vspace{-.5cm}
\end{figure}

\section{More synthetic experiments}

\label{app:synth}

{In this section, we provide additional experimental results on synthetic networks. When we refer to specific $\ergm$s we use the same parameters as in 
\Cref{sec:expsim} in the main text.}

\subsection{Re-estimation after $k$ steps} \label{app:reest}

In the main text, we present 
{\steingen} with re-estimation after one step in the Glauber dynamics chain; {we re-estimate the transition probability when the sampled network differs from the previous network. By construction, the two networks will deviate in one edge indicator, 
{rendering} the re-estimation procedure  to be not  very {computationally} efficient. Hence it may be of interest to re-estimate the transition probability only after $k$ 
different networks have been obtained.}  
Here we investigate the setting where the re-estimation happens after $k $ {changes in the network} ({where networks could be repeated, but not consecutively}), {for networks on 30 vertices}.  
We also consider the setting of {\steingen}\_nr where no re-estimate applies; {one could view this case as $k = \infty$}.

\begin{figure}[tp!]
    \centering
    \begin{center}
   
   \subfigure[E2S Model
   ]{
    \includegraphics[width=0.3\textwidth]{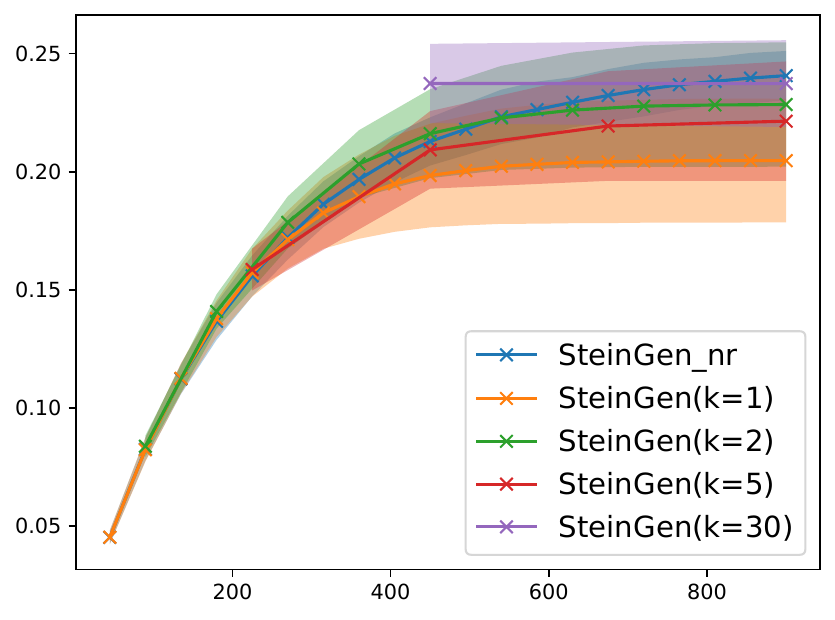}}
    \subfigure[E2ST Model
    ]{
    \includegraphics[width=0.3\textwidth]{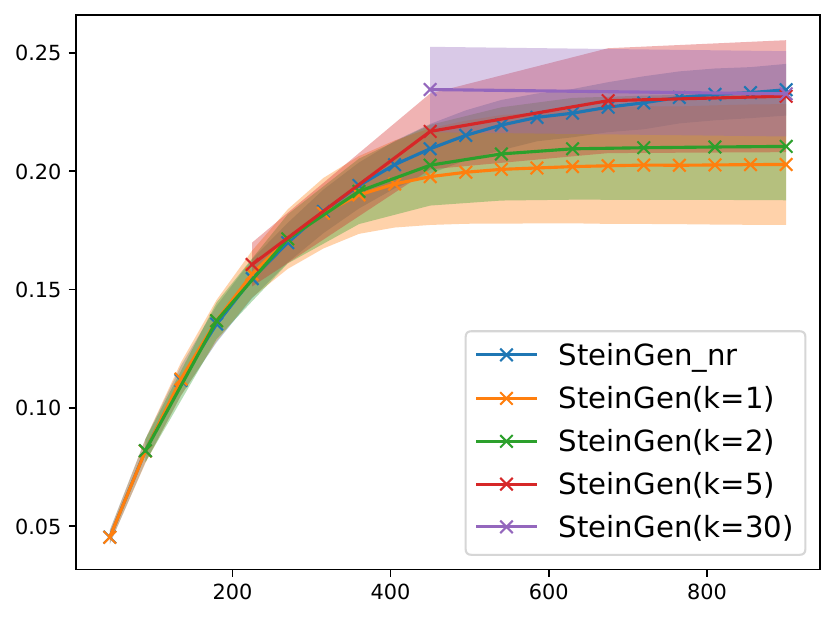}}
    \subfigure[ET Mode 
    ]{
    \includegraphics[width=0.3\textwidth]{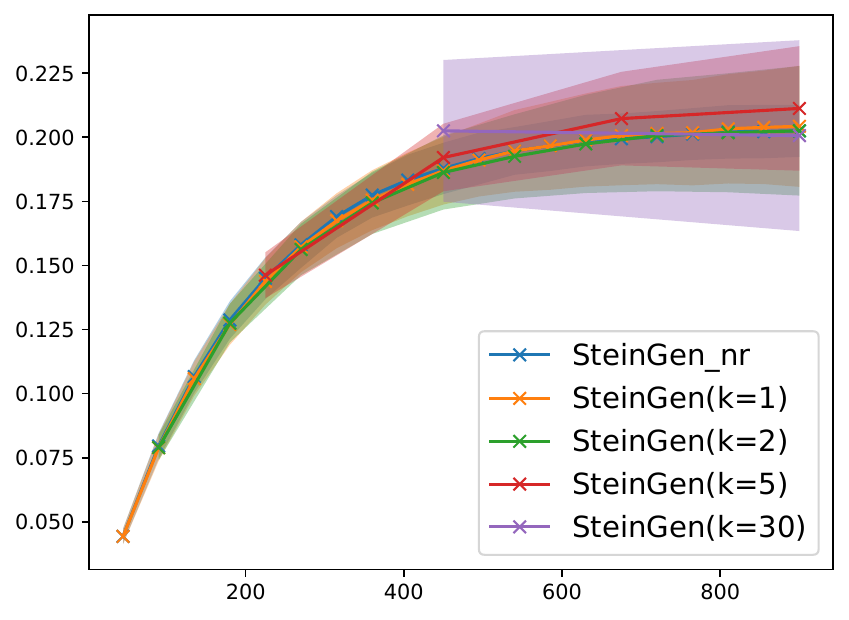}}
    \vspace{-0.25cm}
    \caption{Hamming distance  {$\pm 1$ standard deviation} between generated samples and {an initial network} 
    from each model for different {re-estimation steps, including}
    {\steingen}\_nr as a benchmark. 
    The number of vertices is set to 
    $30$. {For each re-estimation step size $k$ we record the first observations after $k \times n/2$ steps to unify the comparison by accounting for the number of steps in the Glauber process.}
        \label{fig:hamming2}
    }
    \end{center}
   \vspace{-1.cm}
\end{figure}

\begin{figure}[t!]
    \centering
    \begin{center}
    {
{\includegraphics[width=0.98\textwidth]{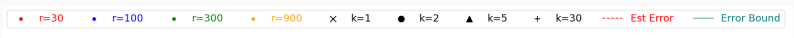}}}
    \vspace{-0.5cm} 
    \subfigure[E2S Model]{
    \includegraphics[width=0.318\textwidth]{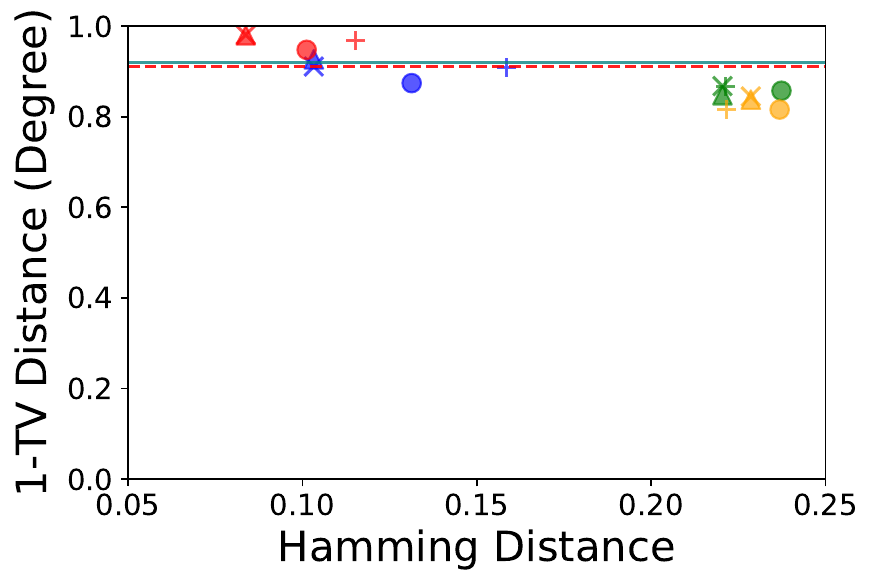}}
    \subfigure[E2ST Model]{
    \includegraphics[width=0.318\textwidth]{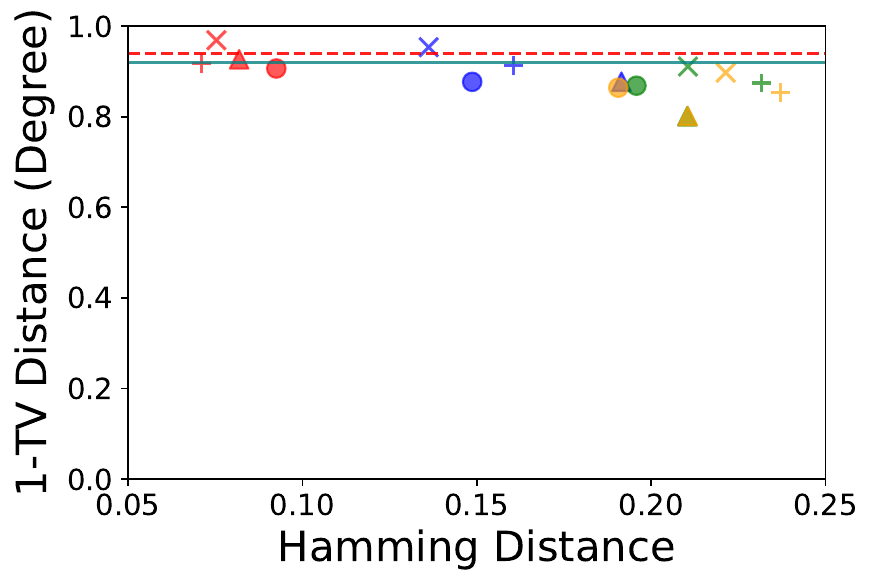}}
    \subfigure[ET Model]{
    \includegraphics[width=0.318\textwidth]{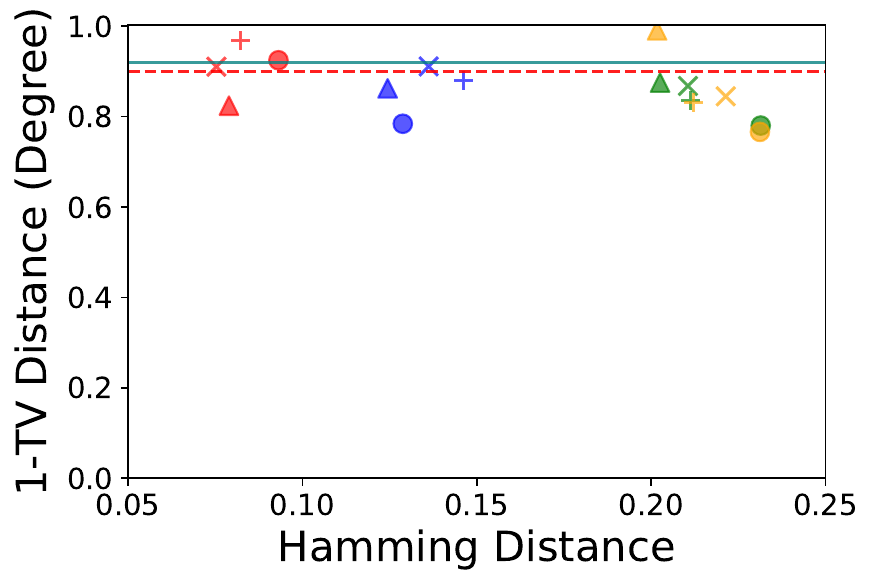}}
    \vspace{-0.cm}
    \caption{Hamming distance versus 1-TV distance {of} degree with generated samples; $r$ is the number of {steps} in {\steingen};
    $k$ is the number of steps 
    {between} re-estimation; $n=30$;
    {Est Error is estimated from simulations; Error Bound is 
    $({n \pi})^{-\frac12}$ in \eqref{eq:dtvbound}.
\label{fig:frontier_different_k}}
   }
\end{center}
\vspace{-1.5cm}
\end{figure}

We plot the Hamming distance for various models, varying the number of re-estimation steps, 
in \Cref{fig:hamming2}. From the plot we see that, with higher number of re-estimation steps $k$,
the Hamming distance converges to the limit faster, which is expected and echoes the behavior of {\steingen}\_nr. Moreover, as $k$ increases, the variance of Hamming distance increases as well, implying an increase in sample variety. 
{\Cref{fig:frontier_different_k} 
shows the relationship between diversity and fidelity; the number of steps in {\steingen} has a more substantial effect on diversity than the number of steps between re-estimation. As expected, re-estimation after every step achieves highest fidelity; there is a trade-off between fidelity and computational efficiency.} 
}

\subsection{Graph generation from multiple network observations}
\label{app:mult}

\begin{wraptable}{r}{0.5\linewidth}
    \centering
     \vspace{-.75cm}
     \caption{Rejection rate for the the gKSS test with $5$ observed network samples; network size $n=30$; test level $\alpha=0.05$. The closest value to the test level is marked in red and the second closest in blue.}
    \label{tab:gkss_multi}
    \vspace{-0.2cm}
    \begin{tabular}{c|cccc}
       \toprule
       {Model} & E2S &  ET & E2ST & ER \\
       \hline
        MPLE & 0.09  & 0.07& 0.08& {\color{blue}0.06}\\
        CD & 0.13 &  0.17& 0.19 & 0.09\\
        MLE & {\color{blue}0.06} & 0.11& 0.07 &0.07\\
        {\steingen}\_nr & 0.03 &  {\color{blue}0.06}& {\color{red}0.05}&{\color{blue}0.06}\\
        {\steingen} & {\color{red}0.04} & {\color{red}0.05}& {\color{blue}0.06} &{\color{red}0.04}\\
        \bottomrule
    \end{tabular}
\end{wraptable}
 

As mentioned in \Cref{sec:conclusion}, 
{\steingen} can 
easily be expanded to multiple graph inputs.
{Heuristically,} 
the estimation of the conditional distribution {should be} improved when multiple input graphs are available, as it can {then} be estimated from the collection of graphs.
{Here we show some experiments to illustrate this extension; we use the same setup as in \Cref{sec:expsim}, on 30 vertices, but now with 5 observed network samples from each model.}
For the parameter estimation counterparts, we use the \texttt{ergm.multi} implementation recently 
{added to} 
the \texttt{statnet} suite \citep{krivitsky2022advanced}.
{To estimate the conditional distribution in \Cref{alg:est_conditional},
if there are $c$ observed network samples $x_1, \ldots, x_c$ from each model, on $n$ vertices each, we denote by $N_i$ the set of pairs of vertices in network $i$, for $i=1, \ldots, c$, so that $|N_i| = N$ for all $i$. Then we estimate 
\[
q(x^{(s,{1})} |{\Delta_s t(x)} =\uk)
= \frac{\sum_{i=1}^c \sum_{s \in N_i} \mathbb{I} (x_i^{s}=1)\mathbb{I} ({\Delta_s t(x)} =\uk) }{\sum_{i=1}^c \sum_{s \in N_i} \mathbb{I} ({\Delta_s t(x)} =\uk) }
\]
where $\mathbb{I}(A)$ is the indicator function of an event $A$ which equals 1 if $A$ holds and 0 otherwise. \\
 Similarly to what was was carried out for} Table \ref{tab:gkss_results}, a gKSS test is performed; the rejection rates are reported in \Cref{tab:gkss_multi}. {Compared to Table \ref{tab:gkss_results}, the rejection rates show a marked improvement for all methods;  for {\steingen} and {\steingen}\textunderscore{nr} they are now very close to the desired 0.05 even for E2ST.  Except on the simple ER network for which {\steingen}\textunderscore{nr} ties with MPLE, {\steingen} and {\steingen}\textunderscore{nr}  again outperform the other methods.}

\subsection{Improving sample quality with 
{\gKSS} 
selection}
\label{app:select}

{The standard {\steingen} procedure may produce a sampled network which may not be very representative of the true network, judged by {\gKSS}.  
As mentioned in \Cref{sec:conclusion}, 
 %
 {\gKSS} 
 could also} be used as a criterion to select samples 
for potential downstream tasks. 
As an illustration, we generate 30 network samples on 30 vertices from the E2S model 
in \Cref{sec:expsim}, {calculate the {\gKSS} for each of these 30 samples, and select the 10 samples with the smallest {\gKSS} 
value. We repeat this experiment $50$ times.} 
}{\Cref{fig:gkss_select_frontier} shows 
{a slight} improvement in fidelity, {but there can be a slight deterioration in} diversity, according to our measures.}

\begin{figure}[t!]
    \centering
    \begin{center}
    {
{\includegraphics[width=0.8\textwidth]{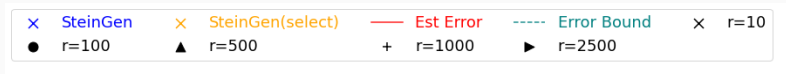}}}
    \subfigure[E2S Model]{
    \includegraphics[width=0.318\textwidth]{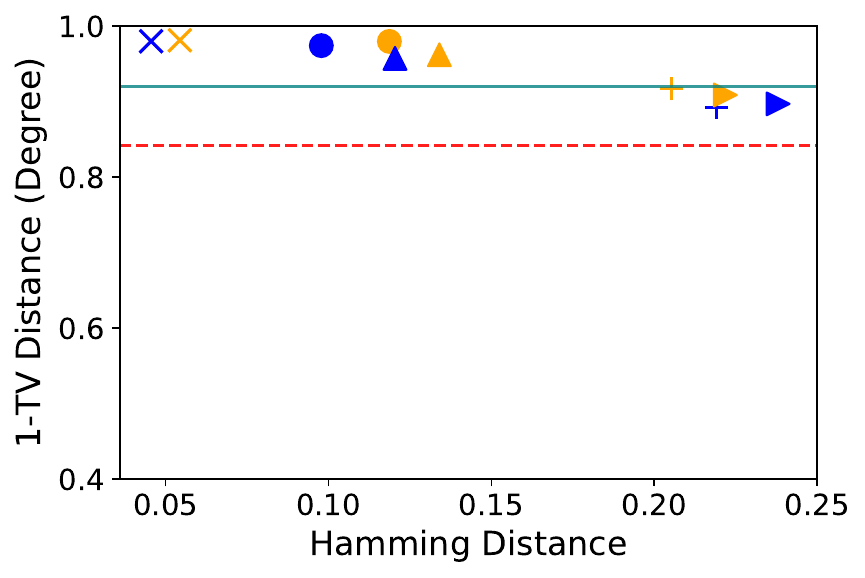}}
    \subfigure[E2ST Model]{
    \includegraphics[width=0.318\textwidth]{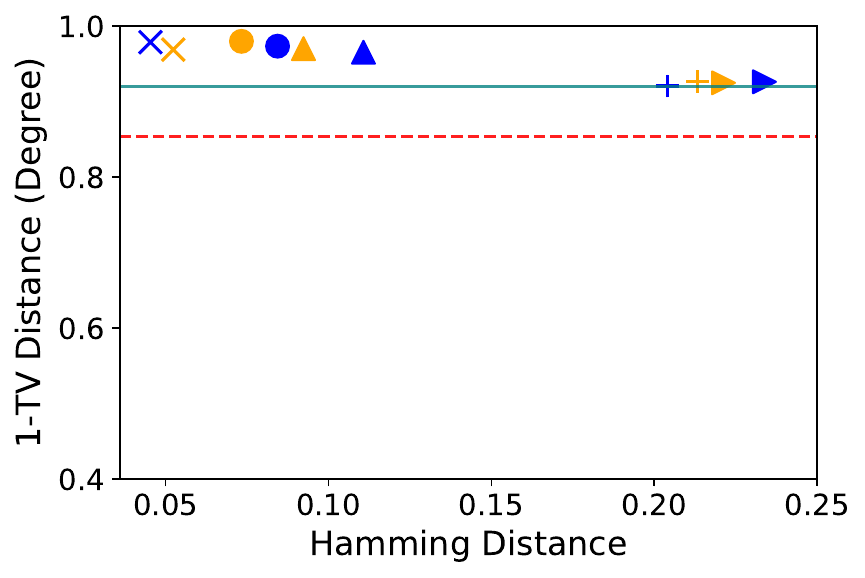}}
    \subfigure[ET Model]{
    \includegraphics[width=0.318\textwidth]{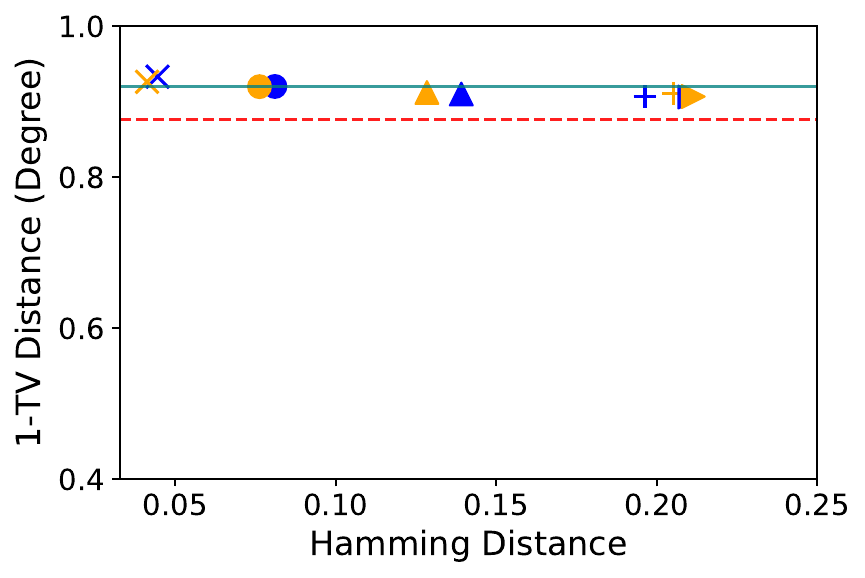}}
    \vspace{-0.5cm}
    \caption{Hamming distance versus 1-TV distance {of} degree using generated samples {with and without batch selection}; $r$ is the number of {steps} in {\steingen};  
    {Est Error, in red, is estimated from simulations, while Error Bound, {in blue,} is the bound $(\sqrt{\pi n})^{-1}$ from \eqref{eq:dtvbound}.
\label{fig:gkss_select_frontier}
   }
   }
\end{center}
\vspace{-1.5cm}
\end{figure}

\section{Additional real data experiments}
\label{app:real}

{Here we report results from experiments on additional real network data. In the absence of a ground truth model, we consider a method as performing well if the observed network statistic in the real network is within two standard deviations of the average in the generated samples. In addition we judge diversity by the Hamming distance to the observed real network; the larger the Hamming distance, the more diverse the samples.}

 

\subsection{Additional results for the teenager network: {kernel choice}} \label{subsec:teenager}
{In this subsection} present further experimental results on the teenager friendship network \citep{steglich2006applying} discussed in \Cref{subsec:real}.

\begin{wraptable}{r}{0.75\linewidth}
    \centering 
    \vspace{-0.8cm}
    \caption{AgraSSt rejection rate at 5\% level 
    with different graph kernel choices for the
    teenager network}
    \label{tab:teenager_kernel}
\begin{tabular}{l|rrrrr}
\toprule
{}  & MPLE & CD & CELL & {\steingen} & {\steingen}\_nr \\
\midrule
WL & 0.68 &  0.92 &  0.12 &  0.06 &  0.08  \\
GVEH & 0.46 &  0.74 &  0.36 &  0.08 &  0.04  \\
SP & 0.34 &  0.62 &  0.10 &  0.02 &  0.04  \\
Const & 0.24 &  0.32 &  0.10 &  0.04 &  0.06  \\
\bottomrule
\end{tabular}
\vspace{-0.25cm}
\end{wraptable} 

We investigate the quality of generated samples from various schemes with AgraSSt using different graph kernels. \textbf{WL}: Weisfeiler-Lehman (WL) graph kernels  \citep{shervashidze2011weisfeiler} with level parameter $3$ as presented in the main  text; \textbf{GVEH} Gaussian Vertex-Edge Histogram kernel \citep{sugiyama2015halting} with unit bandwidth; \textbf{SP}: the Short Path kernel \citep{borgwardt2005shortest}; and \textbf{Const}: the ``constant'' kernel as considered in \citet{weckbecker2022rkhs}. From the rejection rate in \Cref{tab:teenager_kernel}, we see that {\steingen} and {\steingen}\_nr achieve rejection rates {which are} much closer to the significance level $0.05$ compare to MPLE, CD and CELL. {The WL kernel and the constant kernel yield rejection rates which match the 5\% level most closely for {\steingen}, but other kernels perform fairly similarly, confirming the findings in \cite{weckbecker2022rkhs}.}

%


\begin{table}[t!]
    \centering
    \footnotesize
\caption{Statistics for generated samples from  Lazega's lawyer network; $50$ networks are generated from each methods; reporting average(avg) and standard deviation (sd).
    \label{tab:lazega_stats}}
    \vspace{-.25cm}
\begin{tabular}{l|rrrrrrr|c}
\toprule
 & MPLE & CD & MLE & CELL & {NetGAN} & {\steingen} & {\steingen}\_nr & Lazega \\
 \hline
Density(avg) & 0.184 & 0.191 & 0.180 & 0.182 & 0.205 & 0.182 & 0.183 & 0.183 \\
Density(sd) & 0.023 & 0.018 & 0.017 & 0.001 & 0.003 &  0.025 & 0.015 & - \\
\hline
2Star(avg) & 729.0 & 785.4 & 693.1 & 921.3 & 899.2 & 722.3 & 755.0 & 926 \\
2Star(sd) & 173.2 & 143.8 & 126.7 & 23.7 & 38.6 & 133.6 & 82.7 & - \\
\hline
Triangles(avg) & 46.2 & 50.9 & 42.9 & 105.4 & 84.2 & 139.8 & 124.8 & 120 \\
Triangles(sd) & 16.4 & 14.8 & 12.1 & 8.25 & 11.59 & 38.2 & 27.8 & - \\
\hline
SP (avg) & 2.09 & 2.05 & 2.10 & 2.22 & 2.02 & 2.12 & 2.09 & 2.14 \\
SP (sd) & 0.148 & 0.101 & 0.096 & 0.042 & 0.034 & 0.073 & 0.048 & - \\
\hline
LCC (avg) & 35.9 & 36.0 & 35.9 & 35.9 & 34.0 & 36.0 & 36.0 & 34\\
LCC(sd) & 0.180 & 0. & 0.300 & 0.359 & 0. & 0. & 0. & - \\
\hline
Assortat.(avg) & -0.071 & -0.046 & -0.079 & -0.164 & -0.040 & -0.139 & -0.033 & -0.168 \\
Assortat.(sd) & 0.098 & 0.065 & 0.091 & 0.058 & 0.081 & 0.069 & 0.058 & - \\
\hline

Clust.(avg)  &  0.1850 &   0.1911 &   0.1831 &   {\color{red}{0.3429}} & 0.2885 & {\color{blue}{0.3418}}  &   0.4869 &   0.3887 \\
Clust.(sd)  &  0.0281 &  0.0276 &  0.0271 &  0.0222 &  0.0287 & 0.0981 &  0.0903 &  - \\ \hline 
Max deg (avg) &  11.30 &  12.00 &  11.20 &  {\color{red}{15.53}} & {\color{blue}15.67} & 11.60 &  {{12.10}}  &  15.00 \\
Max deg (sd) &  1.2688 &  1.4605 &  0.9451 &  1.0241 & 0.9428 & 1.6248 &  1.3747 &  - \\
\hline 
\hline
AgraSSt(avg) &  0.184 &  0.214 &  0.132 &  0.084 & 0.135 & 0.095 &  0.113 &    0.054 \\
AgraSSt(sd) &  0.182 &  0.067 &  0.102 &  0.034 & 0.158 & 0.077 &  0.066 & -\\
\hline
Hamming(avg) & 0.293 & 0.296 & 0.287 & 0.056 & 0.145 & 0.212 & 0.215 & - \\
Hamming(sd) & 0.019 & 0.017 & 0.018 & 0.007 & 0.008 & 0.010 & 0.012 & - \\
\bottomrule
\end{tabular}
\end{table}

\subsection{Padgett's Florentine marriage network}

\begin{figure}
\centering
\includegraphics[width=0.48\textwidth]{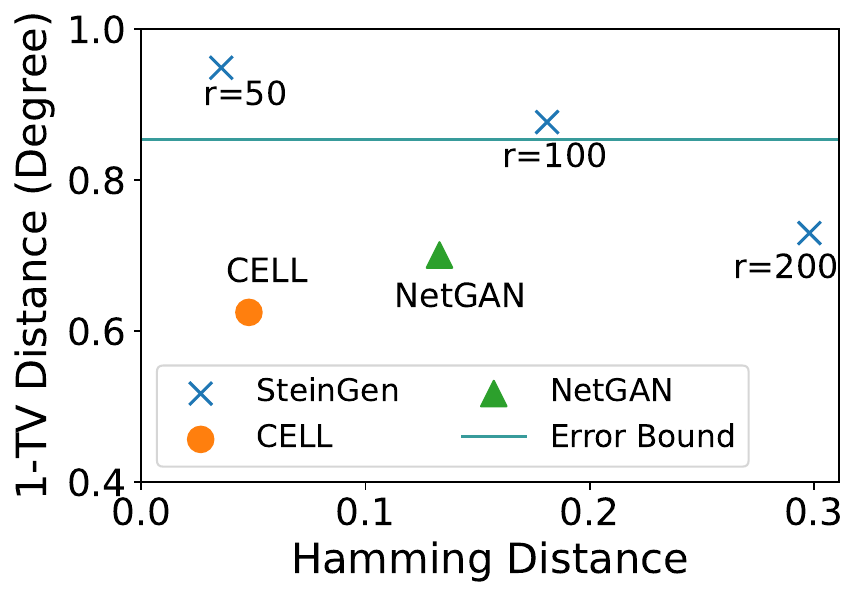}
\caption{Hamming distance versus 1-TV distance {of} 
degree for the Florentine marriage network.
\label{fig:florentine_frontier}}
 \end{figure}

 Padgett’s Florentine marriage network \citep{padgett1993robust}, with 16 vertices representing Florentine  families during the Renaissance and 20 edges representing their  marriage 
ties, 
is a benchmark network for network analysis.
In \cite{reinert2019approximating} and \cite{xu2021stein}, an ER model was considered a good fit.
{Our simulation setup is as in 
\Cref{subsec:real} and we report the same summary statistics in \cref{tab:florentine_stats}. {The  heuristic 
 bound \eqref{eq:dtvbound} 
would give a lower bound on the expected value of $ 0.8589526
$
 for $1 - TV$ Distance.
}} 


{\Cref{fig:florentine_frontier} shows fidelity and diversity for the different methods; here, for large enough $r$, {\steingen} achieves considerably higher diversity, and slightly higher fidelity, than CELL or NetGAN.}



\begin{table}[tp!]
    \centering
\footnotesize
\caption{Statistics for generated samples from {the} Florentine marriage network; reporting average(avg) and standard deviation (sd).}
    \vspace{-0.1cm}
    \label{tab:florentine_stats}
    \begin{tabular}{l|rrrrrrr|c}
\toprule
{} &      MPLE &      CD &       MLE &      CELL & {NetGAN} & {\steingen} &  {\steingen}\_nr & Florentine\\
\midrule
Hamming &  0.268 &  0.253 &  0.257 &  0.048 &  0.133 & 0.301 &     0.268 & \\
AgraSSt&  0.08 &  0.10 &  0.06 &  0.04 & 0.06 & 0.05 &     0.04 & 
pval=0.15\\
\hline
Density (avg) &    0.169 &   0.160 &   0.162 &   0.167 &  0.190 & 0.172 &   0.165  & 0.167\\
Density (sd) &    0.0361 &  0.0379 &   0.0347 &   2.77e-3 &  7.75e-4 &  0.0298 &   0.0254  & - \\
\hline
2Star (avg) &  47.66 &  44.12 &  44.10 &  45.86 & 47.82 & 63.64 &  48.80 & 47\\
2Star (sd) &  20.38 &  21.39 &  17.96 &  3.86 & 4.59 & 38.87 &  15.61 & - \\ \hline 

Triangles (avg) &  2.80  &  2.72  &  2.22  &  2.10  & 2.42 &  5.41  &   4.5  & 3 \\ 
Triangles (sd) &  2.51 &  2.36 &  1.57 &   1.04 &  1.47 & 3.18 &   2.04 & - \\ \hline 
\midrule
SP (avg) &   2.543 &   2.564 &   2.574 &   2.704 &  2.600 & 2.510 &   2.439 &   2.486  \\  
SP (sd) &  0.339 &  0.502 &  0.325 &  0.188 & 0.219 & 0.439 &  0.254 & - \\ \hline
LCC (avg) &  14.60 &  13.86 &  14.20 &  15.96 & 15 & 15.78 &  16.00 &  15 \\
LCC (sd) &  1.70 &  1.91 &  0.943 &  0.101 &  0. & 0.229 &  0.229 & - \\
\hline
Assortat.
(avg) &  -0.141 &  -0.131 &  -0.123 &  -0.328 & -0.249 & -0.093 &   -0.132 &  -0.375 \\
Assortat.
(sd) &  0.141 &  0.156 &  0.158 &  0.114 & 0.105 & 0.107 &  0.124 & - \\
\midrule \hline 
Clust. (avg)  &  {\color{blue}{0.1532}}  &  0.1665 &  {\color{red}{0.1474}}  &  0.1386 & 0.2103 & {\color{blue}{0.1532}}  &  0.1665 &  0.1474 \\
Clust. (std)  &  0.1046 &  0.1098 &  0.0841 &  0.0700 & 0.0945 & 0.1046 &  0.1098 &  - \\ \hline 
Max deg (avg) & 4.980 &  {\color{blue}{5.060}} &  {\color{red}{5.120}} &  6.000 & 5.167 & 4.980 &  {\color{blue}{5.060}} &  5.120 \\
Max deg (std) &   0.1532 &  0.1665 &  0.1474 &  0.1386 & 1.067 & 0.1532&  0.1665 &  - \\
\bottomrule
\end{tabular}
\end{table}

{\Cref{tab:florentine_stats} gives the result from generating 30 samples each for the different network generators. {\steingen} has the largest Hamming distance. While {{\steingen} samples} deviates
from some of the observed network statistics more than the other methods, all observed values {of the sufficient statistics} are well within one standard deviation of the values in the Florentine marriage network.
{We note that the methods based on parameter estimation perform  best for this small benchmark data set.} 

{\subsection{Protein-Protein Interaction (PPI) networks}\label{app:PPI}}
Protein-protein interactions (PPI)  are crucial for various biological processes; {for a survey see for example \cite{silverman2020molecular}. Here we consider two examples, the Epstein-Barr virus and yeast.}  


\paragraph{{An} Epstein-Barr Virus (EBV) network}
We first examine a relatively small PPI network, the Epstein-Barr Virus (EBV) network {used in} 
\citep{ali2014alignment}; see also \citep{hara2022comprehensive}\footnote{The dataset 
can be downloaded from the
\url{https://github.com/alan-turing-institute/network-comparison/blob/master/data/virusppi.rda}.}. 
This network has one connected component that consists of 60 vertices and 208 edges, thus having edge density 0.11751. 

\begin{figure}[h]
\vspace{-0.3cm}
    \centering
    \begin{center}
    \subfigure[$t=$(E, 2S)]{
    \includegraphics[width=0.235\textwidth]{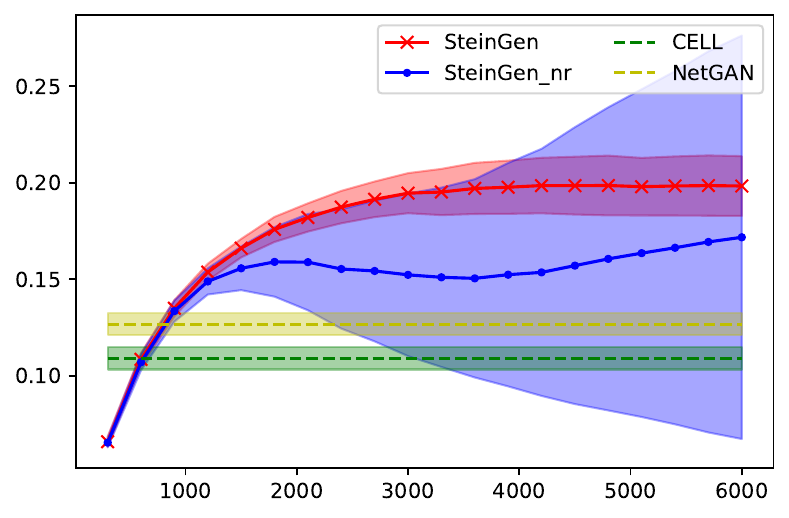}}
    \subfigure[$t$= (E, T)]{
    \includegraphics[width=0.235\textwidth]{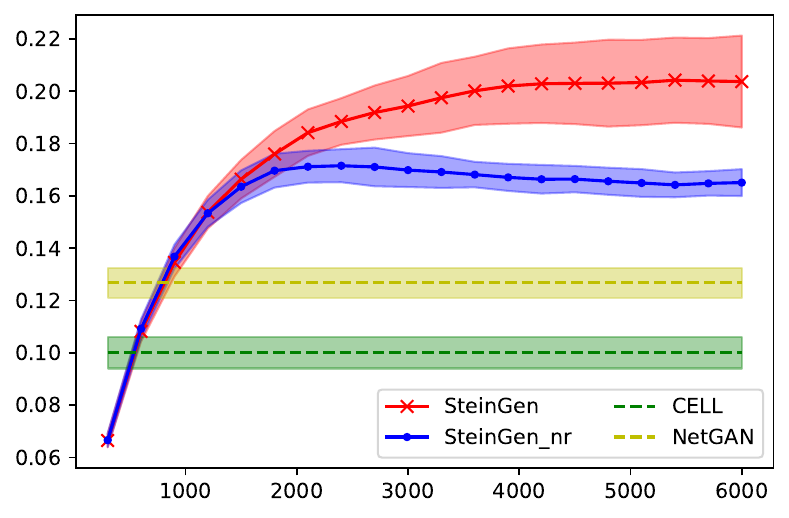}}
    \subfigure[$t$= (E, 2S, T)]{
    \includegraphics[width=0.235\textwidth]{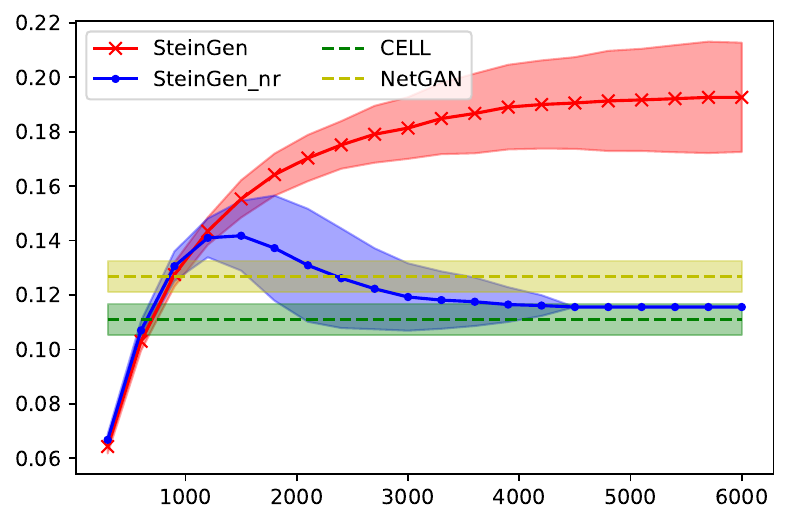}}
     \subfigure[$t$=E]{
    \includegraphics[width=0.235\textwidth]{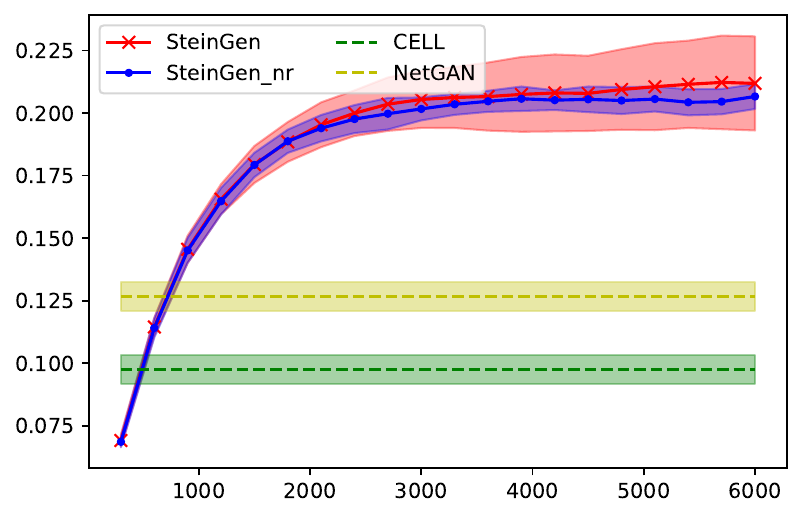}}
    \end{center}
    \vspace{-0.5cm} 
    \caption{Hamming distance for the Epstein-Barr Virus PPI network}
    \label{fig:hamming_ebv}
\end{figure} 
{Using different network statistics $t(x)$ we obtain  the Hamming distance  to the original network in \Cref{fig:hamming_ebv}.
The statistics $t(x)$ used in the models are found in the captions, with $E$ denoting the number of edges, $2S$ the number of 2-stars, and $T$ the number of triangles. We also show the Hamming distance for networks generated by  CELL and NetGAN. While {\steingen} performs similarly across models, achieving the largest Hamming distance, {\steingen}\_nr is more erratic in models which include the number of 2-stars.} The average and standard deviation are taken over $50$ network samples from each method. {The heuristic bound \eqref{eq:dtvbound} 
would give a lower bound on the expected value of $ 0.9271634
$
 for $1 - TV$ Distance.}

{\Cref{tab:EBVsumm} shows various network summaries for generated networks from {\steingen} and {\steingen}\_nr, {with $t(x)$ the number of edges and 2-stars, as well as} CELL and NetGAN.} 
The network statistics from {\steingen} samples are  closest or second closest to the observed EBV samples, with a larger standard deviation (std) than CELL or NetGAN samples.
The achieved Hamming distances of 
{\steingen} and {\steingen}\_nr exceed both CELL and NetGAN.

\begin{figure}[h]
    \centering
\includegraphics[width=0.48\textwidth]{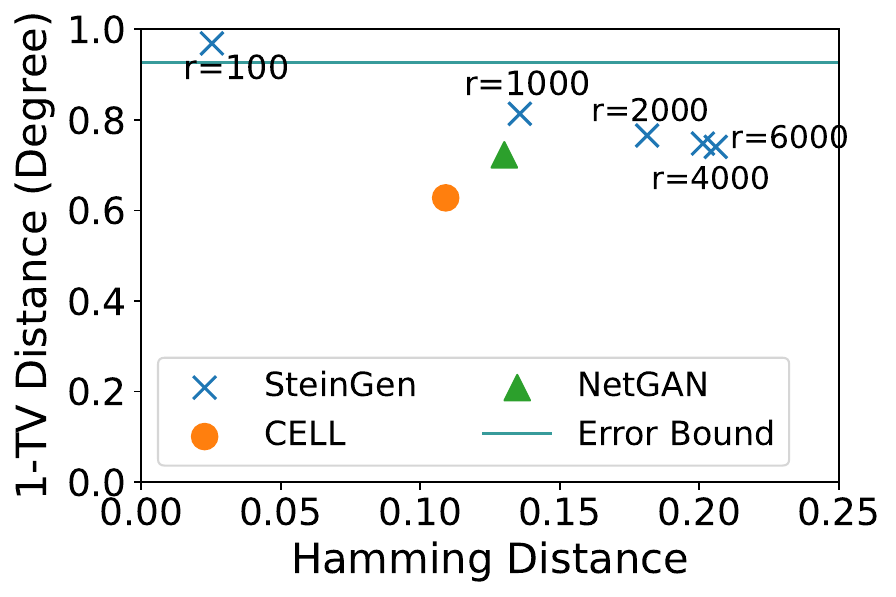}
    \caption{Hamming distance versus 1-TV distance {of} degree for the EBV  network; $r$ is the number of {steps} in {\steingen}. 
    }
\label{fig:frontier-ebv}
\end{figure}

We show the {corresponding} fidelity-diversity trade-off plot in \Cref{fig:frontier-ebv}. 
The empirical TV distance and Hamming distance are computed from averaging over $50$ samples from each generation method. 
{W}ith $r>1000$, the {\steingen} achieves higher fidelity while keeping better diversity compare to CELL and NetGAN. Moreover, the TV distance does not change much from $r=2000$ to $r=6000$.

\begin{table}[t!]
\small
   \caption{EBV network statistics 
   }
    \label{tab:EBVsumm}

    \vspace{0.2cm}

\begin{tabular}{l|rrrrrrrr}
\toprule
{} &  density & 2Star & Triangle & Short.\,Path & LCC & Assortat. & Clust. & Max(deg) \\
\midrule
{\steingen} &  {\color{blue}{0.1170}}  &  {\color{blue}{1956}} &   {\color{red}{214.3}} &  {\color{blue}{2.318}}  &  {\color{red}{60.0}} & {\color{red}{-0.1664}}  &  {\color{red}{0.3524}} &  19.02 
\\
std &  0.0260 &  418.7 &  162.4 &  0.1447 &  0.0 &  0.0586 &  0.1248 &  5.863 
\\
\hline 
{\steingen}\_nr  &   0.1627 &  {{1913}}  &  567.5 &  2.285 &  59.82 & -0.2761 &  0.5699 &  {\color{red}{25.78}} 
\\
std  &  0.0281 &  520.2 &  206.9 &  0.1152 &  0.4331 &  0.0400 &  0.1180 &  3.651 
\\ \hline 
CELL  &  {\color{red}{0.1176}} &  1865 &  {116.0} &  {\color{red}{2.335}} &  {\color{red}{60.0}}  & {-0.1469}  &  {{0.1862}} &  {{21.04}} 
\\
std &  1.388e-5 &  66.78 &  15.14 &  0.0323 &  0.0 &  0.0581 &  0.0205 &  2.441 
\\
\hline 
NetGAN &  {\color{red}{0.1174}} &  {\color{red}{1981}} &  {\color{blue}{146.5}} &  {\color{blue}{2.318}}  &  {\color{red}{60.0}}  & {\color{blue}{-0.1521}}  &  {\color{blue}{0.2216}} &  {\color{blue}{22.66}} \\
std &  1.388e-6 &  61.38 &  13.94 &  0.03699 &  0.0 &  0.06409 &  0.01745 &  2.405 \\
\midrule
\hline
EBV &  0.1175 &  2277 &  209.0 &  2.442 &  60.0 & -0.1930 &  0.2753 &  27.0 
\\
\bottomrule
\end{tabular} 
\vspace{-0.15cm}
 
\end{table}

\begin{table}[t!]
    \centering
    \caption{AgraSSt rejection rates for {\steingen} on the EBV network with different $t(x)$
    }
    \vspace{0.2cm}
    \label{tab:EBV_agrasst}
    \begin{tabular}{c|cccc}
    \toprule
        {} & E + 2S & E + T & E + 2S + T & E(Bernoulli) \\
        \midrule
        {\steingen} & {\color{red}0.06} & 0.18 & 0.12 & {\color{red}0.04} \\
        {\steingen}\_nr & 0.32 & 0.58& 0.38 & {\color{red}0.02} \\
         CELL & 0.24 & 0.24 & 0.22 & 0.20\\
         NetGAN & 0.18 & 0.20 & 0.20 & 0.18\\
        \bottomrule
    \end{tabular}
\vspace{-0.15cm}
\end{table}

\Cref{tab:EBV_agrasst} {gives the rejection rates of} AgraSSt tests with different choice of network statistics $t(x)$, {based on} 
50 trials. From the table, we see that the {\steingen} {samples based on the edges and 2-stars ($E+2S$)} and both {\steingen} and {\steingen}\_nr samples based on the Bernoulli model  ($E$) tend to be not rejected at $5\%$ significance level. The rejection rate for CELL and NetGAN samples tend to be higher.

\paragraph{A PPI network {for yeast}}

{Finally we} examine a relatively larger scale {standard} PPI network, {that of yeast based on}  \cite{von2002comparative}\footnote{The data was adapted from the \texttt{R} package \texttt{igraphdata}, originally downloaded from 
\url{http://www.nature.com/nature/journal/v417/n6887/suppinfo/
nature750.html}.
} The network contains 2617 vertices and 11855 edges. We use the largest component (for CELL and NetGAN comparison) as our observed network, which contains 2375 vertices and 11693 edges; giving the edge density of 0.004148. {The bound \eqref{eq:dtvbound} 
would give a heuristic lower bound on the expected value of $ 0.9884231
$
 for $1 - TV$ Distance.} 
{For both PPI networks, the heuristic bound  is much closer to 1, compared to the 1-TV distance of generated samples, indicating that perhaps the independence assumptions in the heuristic are violated.}
 \begin{figure}
 \centering
\includegraphics[width=0.48\textwidth]{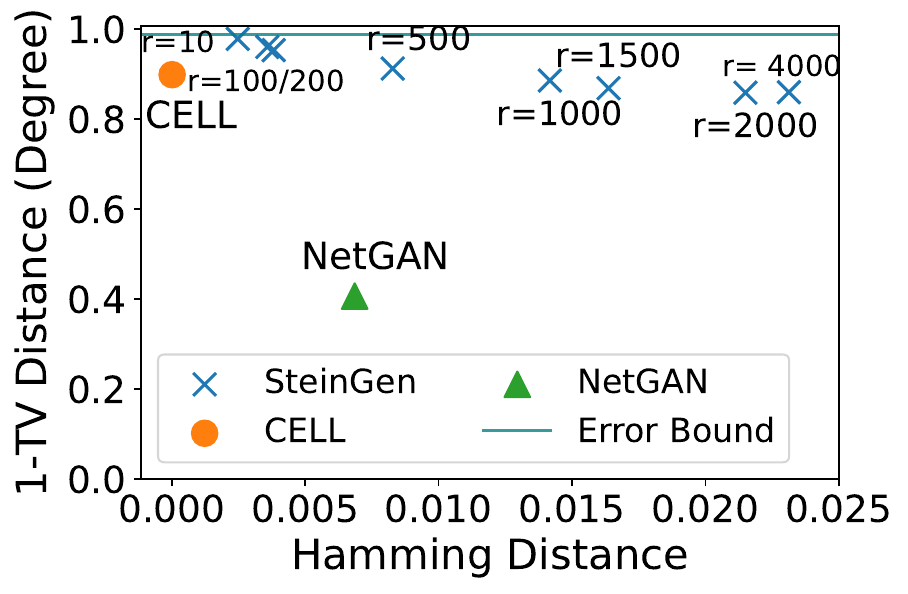}
    \caption{Hamming distance versus 1-TV distance {of} degree for the yeast PPI network; $r$ is the number of {steps} in {\steingen}. 
    \label{fig:frontier-PPI}}
     \end{figure}
{Here the theoretical guideline for the choice of $r$, namely $r = N \log N + \gamma N + 0.5$, yields $r= 43,496,711$. For computational reasons we chose  much smaller values of $r$.}
We report the fidelity-diversity trade-off in \Cref{fig:frontier-PPI}, where the conditional distribution is estimated using edge and 2-stars. From the plot, we see that the {\steingen} method already achieves higher fidelity than  NetGAN, while having larger diversity compared to the CELL and NetGAN samples when the number of {steps}  $r > 500$, which is relatively moderate compared to the large graph size of $2375$. When $r$ increases, {\steingen} samples achieve higher diversity with minimal sacrifice in sample fidelity.

\begin{table}[t!]
\scriptsize
   \caption{PPI network statistics}
    \label{tab:PPIsumm}

    \vspace{0.25cm}

\begin{tabular}{l|rrrrrrrr}
\toprule
{} &  density{($10^{-3}$)} & 2Star{($10^{2}$)} & Triangle(10) & SP{($10^{-3}$)} & LCC & Assortat{($10^{-4}$)}. & Clust.{($10^{-4}$)} & Max(deg) \\
\midrule
{\steingen}  &  {\color{blue}4.158} &  {\color{blue}3884} &  {\color{red}6070} &  {\color{red}5076} &  {\color{red}2375} &  {\color{red}4527} &  {\color{red}4688} &  {\color{red}118.0} \\
std & 6.812e-04 &    1.241 &    0.7980 &  6.149 &     0. &  1.900 &  1.230 &  0. \\
\midrule
{\steingen}\_nr &  4.159 &  {\color{red}3885} &  {\color{blue}6072} &  {\color{blue}5067} &  {\color{red}2375} &  {\color{blue}4521} &  {\color{blue}4689} &  {\color{red}118.0} \\
std & 8.128e-04 &    2.896 &   2.517 &  10.01 &     0. &  5.750 &  3.890 &  0. \\
\midrule
CELL &  {\color{red}4.148} &  3153 &  3034 &  4523 &  {\color{blue}2372} &  3761 &  2887 &  {\color{blue}102.5} \\
std &  0. &   19.63 &  46.98 &  25.80 &     0.8292 &  101.5 & 36.60 &  1.118 \\
\midrule
NetGAN &  3.416 &  1080 &    12.23 &  3821 &  2617 &  30.67 &  33.98 &   22.67\\
std & 1.479 &  467.7 &   5.321 &  165.4 &  1133 &  68.15 &  14.78 &  9.849\\
\midrule
\midrule
PPI & 4.148 &  3885 &  6070 &  5096 &  2375 &  4539 &  4687 &  118.0 \\
\bottomrule
\end{tabular} 
 
\end{table}

\end{document}